\documentclass[letterpaper, 10 pt, conference]{ieeeconf}  

\IEEEoverridecommandlockouts                              
\usepackage{verbatim}
\usepackage{graphics} 
\usepackage{amsmath,amssymb,stmaryrd,mathtools, amsfonts}
\usepackage{multirow}
\usepackage{hhline}
\usepackage{algorithm}
\usepackage{footmisc}
\usepackage[noend]{algpseudocode}
\usepackage{subcaption}
\PassOptionsToPackage{hyphens}{url}\usepackage{hyperref} 
\usepackage{dsfont}
\usepackage{algorithmicx}
\usepackage{multirow}
\usepackage[utf8]{inputenc}
\usepackage{import}
\usepackage{graphicx, graphics}
\usepackage{caption}
\usepackage{subcaption}
\usepackage{array}
\usepackage{color}
\usepackage[table]{xcolor}
\usepackage{siunitx}
\usepackage[short]{optidef}
\usepackage{float}
\definecolor{myedit}{rgb}{1.0,0,0}
\definecolor{mytodo}{rgb}{0,0.0,1.0}
\definecolor{lightgray}{rgb}{0.8, 0.8, 0.8}

\newtheorem{prop}{Proposition}

\newcommand{\M}[1]{\underset{#1}{\text{minimize}}\;}

\newcommand{\mobility}[0]{$\tilde{\mathcal{T}}_{episode}$ & $\Delta\tau $ (m)      
                    }
\newcommand{\conservatism}[0]
                    {$\mathcal{F}$ (\%) 
                    & $\bar d_{min}$ (m)}
\newcommand{\comfort}[0]{$\tilde{\mathcal{A}}_{lat}$ 
                        & $\bar{\mathcal{J}}_{long}$ ($\frac{m}{s^3}$)
                        & $\bar{\mathcal{J}}_{lat}$ ($\frac{m}{s^3}$)
                        }
\newcommand{\efficiency}[0]{$\bar{\mathcal{T}}_{solve}$ (ms)}

\title{\LARGE \bf 
Predictive Control for Autonomous Driving with Uncertain, Multi-modal Predictions
}

\author{Siddharth H. Nair$^{\star}$, Hotae Lee$^{\star}$, Eunhyek Joa$^{\star}$,  Yan Wang, H. Eric Tseng, Francesco Borrelli
\thanks{$^\star{}$Indicates equal contribution.
SHN, HL, EJ and FB are with the Model Predictive Control Laboratory, UC Berkeley. 
YW, HET are with Ford Research and Advanced Engineering. E-mails for correspondence:
\{siddharth\_nair, hotae.lee, e.joa, fborrelli\}@berkeley.edu.}
}

\begin{document}

\maketitle
\thispagestyle{empty}
\pagestyle{empty}
\begin{abstract}
We propose a Stochastic MPC (SMPC) formulation for path planning with autonomous vehicles in scenarios involving multiple agents with multi-modal predictions. The multi-modal predictions capture the uncertainty of urban driving in distinct modes/maneuvers (e.g., yield, keep speed) and driving trajectories (e.g., speed, turning radius), which are incorporated for multi-modal collision avoidance chance constraints for path planning. In the presence of multi-modal uncertainties, it is challenging to reliably compute feasible path planning solutions at real-time frequencies ({\small{$\geq 10 Hz$}}). Our main technological contribution is a convex SMPC formulation that simultaneously (1) optimizes over parameterized feedback policies and (2) allocates risk levels for each mode of the prediction. The use of feedback policies and risk allocation enhances the feasibility and performance of the SMPC formulation against multi-modal predictions with large uncertainty.  We evaluate our approach via simulations and road experiments with a full-scale vehicle interacting in closed-loop with virtual vehicles. We consider distinct, multi-modal driving scenarios: 1) Negotiating a traffic light and a fast, tailgating agent, 2) Executing an unprotected left turn at a traffic intersection, and 3) Changing lanes in the presence of multiple agents. For all of these scenarios, our approach reliably computes multi-modal solutions to the path-planning problem at real-time frequencies.  

\end{abstract}

\section{Introduction}
\label{sec:introduction}
\subsection*{Motivation}

Autonomous vehicle technologies have seen a surge in popularity over the last decade, with the potential to improve the flow of traffic, safety, and fuel efficiency \cite{nhtsa}.
In the upcoming decade, we can expect a rise in the number of autonomous vehicles, making it increasingly common for them to navigate roads with mixed traffic with vehicles of varying automation levels \cite{mcksy}.
While existing technology is being gradually introduced into scenarios such as highway driving \cite{chae2021design} and low-speed parking \cite{zhang2018autonomous} where other road users' intents are relatively easy to infer, autonomous driving in mixed traffic scenarios such as urban road driving and merging is an open challenge because of the
variability in the possible behaviors of the surrounding agents \cite{nhtsa_intersection}, \cite{wei2021autonomous}. 
To address this difficulty, significant research has been devoted to modeling these agent predictions as multi-modal distributions \cite{imm_bar_shalom_1988, multipath_2019, trajectron_2020}. Such models capture uncertainty in both high-level decisions (desired route) and low-level executions (agent position, heading, speed).\\
\begin{figure}[h]
    \centering
    \begin{subfigure}{0.8\columnwidth}
    \centering
    \includegraphics[width=\columnwidth]{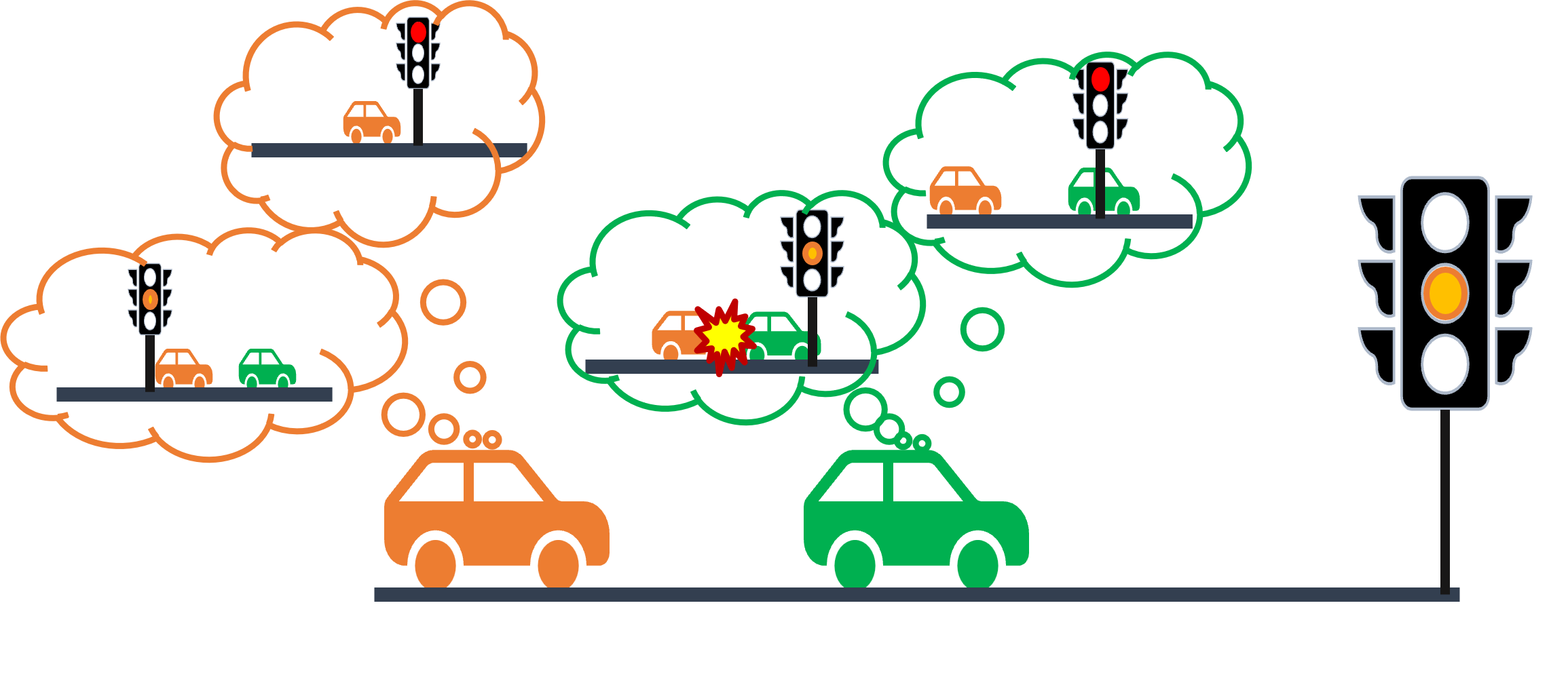}
    \caption{{\small{The EV (green) must decide whether to cross a Traffic Light (TL) before it turns red or come to a stop, while managing the headway for the tailgating TV behind. The discrete modes are given as $\{$TL goes red, TL stays yellow$\}\times\{$TV stops for TL, TV doesn't stop for TL$\}$. Inspired by the dashcam footage: {\footnotesize{\url{https://youtu.be/i3pvrpKDjRQ}}}.}}}\label{fig:TL}
    \end{subfigure} \hfill %
    \begin{subfigure}{0.8\columnwidth}
    \centering
    \includegraphics[width=\columnwidth]{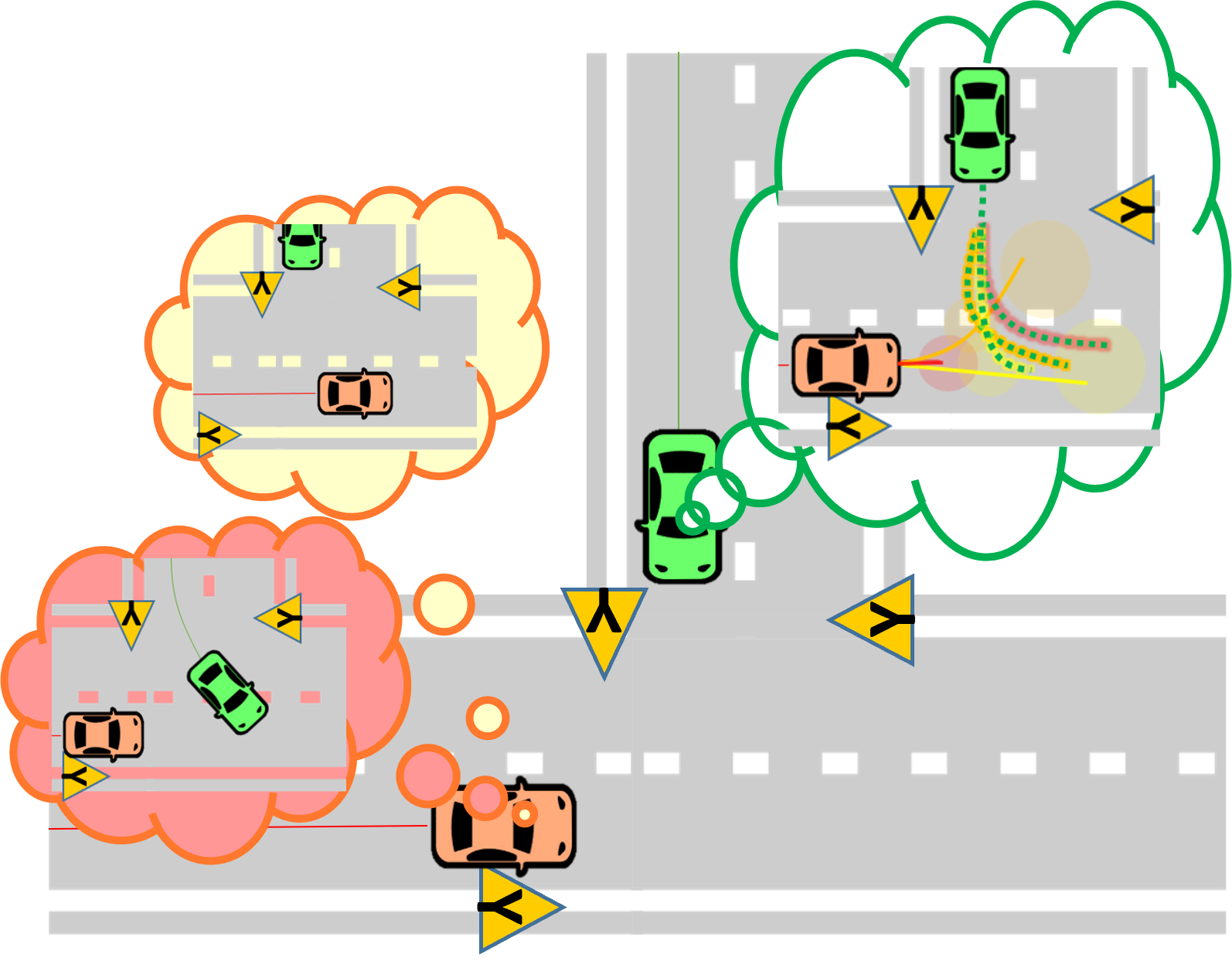}
    \caption{{\small{The EV (green) must find a feasible solution to make a left turn in the presence of an oncoming TV while accounting for its multi-modal behavior: $\{${\color{red}{give EV right-of-way}}, {\color{yellow}{go straight}} and {\color{orange}{turn left}}$\}$. The EV computes a \textit{policy tree} to address the multi-modal uncertainty. Inspired by the footage: {\footnotesize{\url{https://shorturl.at/aeJ12}}}. }}}\label{fig:UL} 
    \end{subfigure}
    \caption{{\small{Multi-modal planning in urban driving.
}}}   
    \vskip -0.9cm
\end{figure}

The focus of this work is to incorporate these multi-modal distributions for the surrounding agents (called Target Vehicles, or TVs) into a planning framework for the autonomous agent (called Ego Vehicle, or EV). We investigate the planning problem in the context of constrained optimal control and use Model Predictive Control (MPC) for computing feedback control policies. The main challenge in designing MPC for effectively addressing the multi-modal predictions is to find a good balance between \textit{performance}, \textit{safety}, and \textit{computational efficiency}. Consider the situation in Fig.~\ref{fig:TL}, where the EV is approaching a traffic light with a tailgating TV behind. A \textit{performant} MPC design would enable the EV to assess the risk associated with the multi-modalities of the TV and traffic light along the planning horizon so that the EV is able to cross the yellow light or stop at the red light. For ensuring \textit{safety}, the EV must also manage a safe distance ahead of the TV despite the uncertain predictions. A conservative MPC design would either fail to find a feasible solution in the presence of large uncertainty, or sacrifice performance for safety by always choosing to stop. Prior works \cite{batkovic2020robust, chen2022interactive, nair2022stochastic} show that planning using trees or feedback policies over the multi-modal distribution is effective for reliably finding high-quality solutions. However, optimization over policies is infinite-dimensional in general and hence, \textit{computationally expensive} for real-time control. In this work, we propose a Stochastic Model Predictive Control (SMPC) framework that incorporates multi-modal predictions of agents to enforce probabilistic collision avoidance and state-input constraints. 

\subsection*{Contributions}
Our main contributions are summarised as follows. First, we propose a convex formulation for Stochastic MPC that optimizes over tree-structured feedback policies for multi-modal predictions specified as Gaussian Mixture Models (GMMs). The policy parameterization is designed to receive feedback over both discrete modes and continuous observations of the TVs' states. Our formulation also includes a novel multi-modal chance constraint reformulation that simultaneously allocates risk levels for the various modes based on their probabilities. Second, we evaluate our approach in various autonomous driving scenarios via simulations and hardware experiments. We demonstrate our SMPC via a hardware experiment for a lane change scenario, characterized by the presence of two TVs with uncertain, bi-modal predictions. Our findings indicate that the proposed approach exhibits a significant reduction in conservatism when compared to the conventional approaches that optimize over open-loop control sequences. Additionally, we show adaptability to variable probabilities of the modes of the TVs.

\subsection*{Related work}
There is a large body of work focusing on the application of SMPC to autonomous driving for motion planning \cite{planning_and_dm_for_avs_2018, brudigam2021stochastic, rosolia2018data}. A typical SMPC algorithm involves a finite horizon stochastic optimal control problem that is solved as a chance-constrained optimization problem in receding horizon fashion \cite{mesbah2016stochastic}. 
In the context of motion planning, SMPC has been used for imposing chance-constrained collision avoidance to account for uncertainty in vehicles' predictions in applications such as autonomous lane change \cite{carvalho2014stochastic, gray2013stochastic}, cruise control \cite{moser2017flexible} and platooning \cite{causevic2020information}. Besides the modeling aspects specific to the applications, SMPC formulations differ in the uncertainty descriptions and the deterministic reformulations of the chance constraints. Gaussian distributions are a common modeling choice for the uncertainty owing to the invariance to affine transformations and closed-form expressions for affine chance constraints \cite{brudigam2021stochastic, brudigam2021gaussian}. For arbitrary distributions that can be sampled efficiently, randomized approaches \cite{calafiore2006scenario} are used to obtain high-confidence reformulations of the chance constraints \cite{de2021scenario}. Distributionally robust formulations are also becoming increasingly popular to improve robustness to distributional assumptions \cite{wei2022moving, nair2022collision, schuurmans2023safe}, by robustifying the chance constraint against a set of distributions. For urban driving scenarios where the surrounding agents' predictions are multi-modal and temporally linked, mixture models \cite{zhou2018joint, wang2020non, ren2022chance, nair2022savec, zhou2023interaction} and scenario trees \cite{batkovic2020robust, chen2022interactive, oliveira2023interaction} offer convenient structure that can be exploited in the SMPC.\\

To enhance the feasibility of the SMPC optimization problem, the scenario-tree-based MPC formulations \cite{batkovic2020robust, chen2022interactive, oliveira2023interaction} optimize over policies along the prediction horizon. The policies inherit a tree structure from the scenario tree, to encode feedback over the uncertainty realization. This adds flexibility to find feasible solutions due to the ability to react to different realizations of the vehicles' trajectory predictions along the prediction horizon. Mixture models like Gaussian Mixture Models (GMMs), are more memory-efficient representations for the multi-modal uncertainty by using discrete random variables to capture distinct modalities and continuous variables to capture the spread within each mode. In order to handle multi-modal predictions (specifically GMMs), the works \cite{zhou2018joint}, \cite{wang2020non} propose nonlinear SMPC algorithms that suitably reformulate the collision avoidance chance constraint for all possible modes. However, a non-convex optimization problem is formulated to find a single open-loop input sequence that satisfies the collision avoidance chance constraints for all modes and possible evolutions of the target vehicles over the prediction horizon given by the GMM. To remedy the conservativism of this approach, we proposed a convex SMPC formulation \cite{nair2022stochastic} that optimizes over a parametric, mixed-feedback policy architecture to enhance the feasibility of the SMPC optimization problem. Inspired by scenario trees, our policies use feedback from the discrete mode realization but also incorporate feedback from the continuous realizations of the agents' positions. \\

In this work, we build on our SMPC formulation \cite{nair2022stochastic}  in two directions. First, we generalize our policy parameterization for multi-modal predictions with arbitrary rooted tree structure. Our earlier approach assumed a simplistic tree structure with a single branching point. Second, we derive a convex reformulation of the multi-modal chance constraints to simultaneously allocate risk levels for each mode and find feasible policy parameters. Compared to our previous approach, this further enhances the feasibility of the SMPC optimization problem by exploiting the varying probabilities of the discrete modes to allocate risk levels for each mode to satisfy the multi-modal chance constraint. In the context of SMPC, risk allocation has been studied for joint chance constraints \cite{ono2008iterative, sivaramakrishnan2020convexified}. In \cite{ono2008iterative}, the risk levels are iteratively optimized by alternating optimization of risk levels and SMPC optimal control problem, whereas \cite{sivaramakrishnan2020convexified} proposes a convex formulation for simultaneous risk allocation and optimization over open-loop control sequences. The simultaneous optimization in \cite{sivaramakrishnan2020convexified} overcomes the computational overhead of the iterative optimization approach for real-time applications, but is not applicable when the SMPC optimizes over policies.\\

Multiple works dealing with multi-traffic scenarios are investigated in simulation settings \cite{brudigam2021stochastic, batkovic2020robust, oliveira2023interaction, liu2022interaction}.
Validating developed autonomous vehicle technology through vehicle tests is desirable to demonstrate its capabilities and effectiveness in real-world setups.
However, investigating algorithms in real-world multi-traffic scenarios is challenging due to safety concerns and legal issues. 
In order to mitigate these challenges associated with evaluating the performance of the proposed control system, a Vehicle-in-the-Loop (VIL) setup is adopted \cite{ma2018VIL, tettamanti2018vehicle}.
In the VIL setup, the Ego Vehicle (EV) is driven in the real world, in closed-loop with the proposed control algorithm while interacting with virtual surrounding vehicles which are simulated using a microscopic simulator \cite{ard2023VILCAV, joa2023energy}.
This allows us to test the performance of the proposed control system in a more realistic yet safe environment, while still leveraging the benefits of simulation for perception and driving environment construction.

\section{Problem Formulation}
\label{sec:prblm_f}
In this section we formally cast the problem of designing SMPC in the context of autonomous driving.

\subsection{Preliminaries}

\subsubsection*{Notation} 
The index set $\{k_1,k_1+1,\dots, k_2\}$ is denoted by $\mathbb{I}_{k_1}^{k_2}$. The cardinality of a discrete set $\mathcal{S}$ is denoted by $|\mathcal{S}|$ (e.g., $|\mathbb{I}_{k_1}^{k_2}|=k_2-k_1+1$). We denote $\Vert\cdot\Vert$ by the Euclidean norm and $\Vert x\Vert_{M}=\Vert \sqrt{M}x\Vert$ for some $M\succ 0$. 
%
%
The binary operator $\otimes$ denotes the Kronecker product.   
The partial derivative of function $f(x,u)$ with respect to $x$ at $(x,u)=(x_0,u_0)$ is denoted by $\partial_x f(x_0,u_0)$. \\


\subsubsection*{EV modeling}
We model the dynamics of the EV in the Frenet frame moving
along a curve $\gamma(s)=[\bar{X}(s), \bar{Y}(s), \bar{\psi}(s)]$ parameterised by the arc length $s$, which describes the position and heading of the centerline of a lane in the road \cite{fork2021models}. 
Let $x_t=[s_t,\ e_{y,t},\ e_{\psi,t},\ v_t]^\top$ be the state of the EV at time $t$ where
 $s_t, e_{y,t}, e_{\psi,t}$ are the arc length, lateral offset and relative heading with respect to the centerline $\gamma(\cdot)$, and $v_t$ is the EV's speed. Then the dynamics of the EV can be described as 
 {\small{
 \begin{align} \label{eq: bicycle model}
     \dot{x}_t=\begin{bmatrix}\dfrac{v_tcos(e_{\psi,t})}{1-e_{y,t}\kappa(s_t)}\\v_t\sin(e_{\psi,t})\\\dot{\psi}_t-\dfrac{v_tcos(e_{\psi,t})\kappa(s_t)}{1-e_{y,t}\kappa(s_t)}\\a_t\end{bmatrix}
 \end{align}}}
 where  $\kappa(s_t)=\dfrac{d\bar{\psi}(s_t)}{ds}$ describes the curvature of $\gamma(\cdot)$, $a_t$ is the EV's acceleration and $\dot{\psi}_t$ is the EV's global yaw rate. The dynamics of the EV are time-discretized (with any explicit integration scheme) to obtain the model $x_{t+1}=f^{EV}(x_t,u_t)$, with inputs $u_t=[a_t,\ \dot{\psi}_t]$. Given the state $x_t$, the EV's global pose can be obtained via a function $\mathcal{G}^\gamma(\cdot)$, {\small{\begin{align}\begin{bmatrix}
     X_t\\ Y_t\\ \psi_t
 \end{bmatrix}=\mathcal{G}^\gamma(x_t)=\begin{bmatrix}\bar{X}(s_t)-e_{y,t}\sin(\bar{\psi}(s_t))\\ \bar{Y}(s_t)+e_{y,t}\cos(\bar{\psi}(s_t))\\ e_{\psi,t}+\bar{\psi}(s_t)\end{bmatrix}.\end{align}}}
The system state and input constraints are given by polytopic sets which capture vehicle actuation limits and traffic rules,
{\small{\begin{align}
\mathcal{X}=\{x: a^\top_{x,i}x\leq b_{x,i} ~\forall i\in\mathbb{I}_1^{n_X}\},\nonumber\\ 
\mathcal{U}=\{u: a^\top_{u,i}u\leq b_{u,i} ~\forall i\in\mathbb{I}_1^{n_U}\}.
\end{align}}}

We assume a kinematically feasible reference trajectory, {\small{\begin{align}\label{eq:EV_ref_traj}
    \{(x^{ref}_t,u^{ref}_{t})\}_{t=0}^{T}
\end{align}}}is provided for the EV. This serves as the EV's desired trajectory which can be computed offline (or online at lower frequency) accounting for the EV's route, actuation limits, and static environment constraints (like lane boundaries, traffic rules).  However, this reference does not consider the dynamically evolving TVs for real-time obstacle avoidance.\\

\subsubsection*{TV predictions}
Let $n_{TV}$ be the number of TVs in consideration and denote the position of the $i$th TV at time $t$ as $o^i_t=[X^i_t\ Y^i_t]^\top$, and define  $o_t=[o^{1\top}_t,\dots,o^{n_{TV}}_t]^\top$ which stacks the positions of all the TVs. For collision avoidance, we use an off-the-shelf prediction model \cite{multipath_2019, trajectron_2020} trained on traffic datasets \cite{l5kit2020, nuscenes2019} that provides $N-$step predictions of the TVs' positions given by a multi-modal Linear Time-Varying (LTV) model $\forall k\in\mathbb{I}_t^{t+N}, j\in\mathbb{I}_1^J$,
{\small{
\begin{align}\label{eq:TV_preds}
o_{k+1|t,j}=T_{k|t,j}o_{k|t,j}+c_{k|t,j}+n_{k|t,j},
\end{align}}}
where $o_{k|t,j}$ is the prediction of the TVs' positions at time $k$ for mode $j$, $T_{k|t,j},c_{k|t,j}$ are time-varying matrices and vectors for the TVs' prediction for mode $j$, and the process noise is given by $n_{k|t,j}\sim\mathcal{N}(0,\Sigma_{k|t,j})$. The mode $j\in\mathbb{I}_1^J$ captures distinct interactions/maneuvers of the TVs as a group. We denote $p_t=[p_{t,1},..,p_{t,j}]$ as the probability distribution over the modes at time $t$, and $\sigma_t$ to be the true, \textit{unknown} mode.

 
 \subsection{Stochastic Model Predictive Control Formulation}
 \label{sec:smpc_approach}
 We aim to design a computationally efficient feedback control $u_t=\pi_t(x_t, o_t)$ for the EV to track the reference trajectory \eqref{eq:EV_ref_traj}, satisfy state-input constraints and avoid collisions with the TVs by effectively addressing the uncertainty arising from the TVs' multi-modal predictions \eqref{eq:TV_preds}.
     
 We propose a Stochastic Model Predictive Control (SMPC) formulation to compute the feedback control $u_t$. The optimization problem of our SMPC takes the form,

 {\small{
 \begin{mini!}[2]
 {\substack{\{\theta_{t,j}\}_{j=1}^J}}{\sum_{j=1}^J p_{t,j}\mathbb{E}\left[C_t(\mathbf{x}_{t,j},\mathbf{u}_{t,j})\right]\label{opt:obj}}{\label{opt:SMPC_skeleton}}{}
\addConstraint{x_{k+1|t,j}}{=f^{EV}_k(x_{k|t,j}, u_{k|t,j})\label{opt:EV_dyn}}
\addConstraint{o_{k+1|t,j}}{=T_{k|t,j}o_{k|t,j}+c_{k|t,j}+n_{k|t,j}\label{opt:TV_dyn}}
\addConstraint{\mathbb{P}(g_k(x_{k+1|t}, o^i_{k+1|t})\geq 0)}{\geq 1-\epsilon \label{opt:oa_constr}}{}
\addConstraint{(x_{k+1|t,j},u_{k|t,j})\in\mathcal{X}\times\mathcal{U}}{\label{opt:ev_constr}}{}
\addConstraint{\mathbf{u}_{t,j}}{\in\Pi_{\theta_{t,j}}(\mathbf{x}_{t,j},\mathbf{o}_{t,j})\label{opt:gen_pol_class}}
\addConstraint{x_{t|t,j}=x_t,\ u_{t|t,j}=u_{t|t,1},\ o_{t|t,j}=o_t }{\label{opt:init}}{}
\addConstraint{~ \forall i\in\mathbb{I}_1^{n_{TV}}, \forall j\in\mathbb{I}_1^{J}, \forall k\in\mathbb{I}_t^{t+N-1}.}{\nonumber}
\end{mini!}
}}
where $\mathbf{u}_{t,j}=[u_{t|t,j},\dots, u_{t+N-1|t,j}]$, $\mathbf{x}_{t,j}=[x_{t|t,j},\dots, x_{t+N|t,j}]$ and $\mathbf{o}_{t,j}$ (defined similarly to $\mathbf{x}_{t,j}$), denote stacked predictions along the horizon for mode $j$. The SMPC feedback control action is obtained from the optimal solution of \eqref{opt:SMPC_skeleton} 
as
{\small{
\begin{align}\label{eq:SMPC}u_t=\pi_{\mathrm{SMPC}}(x_t,o_t)=u^\star_{t|t,1},\end{align}
}}
where the EV and TV state feedback enters the optimization problem in \eqref{opt:init}.
The function  $C_t(\cdot,\cdot)$ in the objective \eqref{opt:obj}  penalizes the deviation of the EV's trajectory for mode $j$ from the reference \eqref{eq:EV_ref_traj}, and is weighted by the probability of the mode given by $p_{t,j}$.  The collision avoidance constraints are imposed as chance constraints \eqref{opt:oa_constr} along with polytopic state and input constraints $\mathcal{X},\mathcal{U}$ for the EV. The EV's controls along the prediction horizon are given by parameterized policies $\Pi_{\theta_{t,j}}(\mathbf{x}_{t,j},\mathbf{o}_{t,j})$ \eqref{opt:gen_pol_class} that are functions of the EV's and TVs' states, as opposed to open-loop sequences. The policies are multi-modal, which makes the EV's closed-loop trajectories in \eqref{opt:EV_dyn} multi-modal. Consequently, the chance constraints \eqref{opt:oa_constr} are defined over the closed-loop multi-modal distributions, re-written using the law of total probability as: 
{\small{
\begin{align*}
&\sum_{j=1}^Jp_{t,j}\mathbb{P}(g_k(x_{k+1|t,j}, o^i_{k+1|t,j})\geq 0) \geq 1-\epsilon.
\end{align*}
}}Deriving a deterministic reformulation of this chance constraint that is computationally efficient, but not too conservative is the key technical challenge in the SMPC design. Towards addressing this challenge, our SMPC formulation features 1) a novel multi-modal policy parameterization \eqref{opt:gen_pol_class} for shaping the multi-modal closed-loop distribution, 2) a convex inner-approximation technique for the multi-modal chance constraint \eqref{opt:oa_constr} involving mode-dependent, risk levels $r_{t,j}=1-\epsilon_{t,j}$ for $ \epsilon_{t,j}\in[0,1]$ and 3) simultaneous, convex optimization over policy parameters and risk levels for control computation.  These features enable computationally efficient synthesis of \eqref{eq:SMPC} while enhancing the feasibility of \eqref{opt:SMPC_skeleton} by effectively addressing the multi-modal uncertainty.

 

\section{Stochastic MPC with Multi-Modal Predictions}

\label{sec:SMPC}

In this section, we detail our SMPC formulation for the EV to track the reference \eqref{eq:EV_ref_traj} while incorporating multi-modal predictions  \eqref{eq:TV_preds} of the TV for obstacle avoidance. 

\subsection{Vehicles' Prediction Models}
The EV prediction model \eqref{opt:EV_dyn} is a linear time-varying model (LTV), obtained by linearizing $f^{EV}(\cdot)$ about the reference trajectory \eqref{eq:EV_ref_traj}. At time $t$, let $\bar{t}=-t+\arg\min_{k\in\mathbb{I}_0^T} |s_{t}-s^{ref}_k|$ and define $\Delta x_{k|t}=x_{k|t}-x^{ref}_{k+\bar{t}},\ \Delta u_{k|t}=u_{k|t}-u^{ref}_{k+\bar{t}}, \forall k\in\mathbb{I}_t^{t+N-1}$. Then the LTV model is given as
\small
\begin{align}\label{eq:EV_ltv_model}
    &\Delta x_{k+1|t}=A_{k|t}\Delta x_{k|t}+B_{k|t}\Delta u_{k|t}+w_{k|t}\\
    &A_{k|t}=\partial_xf^{EV}(x^{ref}_{k+\bar{t}},u^{ref}_{k+\bar{t}}),\ B_{k|t}=\partial_u f^{EV}(x^{ref}_{k+\bar{t}},u^{ref}_{k+\bar{t}})\nonumber
\end{align}
\normalsize
where the additive process noise $w_{k|t}\sim\mathcal{N}(0,\Sigma_w)$ (i.i.d with respect to $k$) models linearization error and other stochastic noise sources. The polytopic state and input constraints \eqref{opt:ev_constr} are replaced by the chance-constraints $\forall k\in\mathbb{I}_0^{N-1}$,
{\small{
\begin{align}\label{eq:EV_cc}
    &\mathbb{P}((\Delta x_{k+1|t},\Delta u_{k|t})\in\Delta\mathcal{X}_k\times\Delta\mathcal{U}_k )\geq 1-\epsilon,\\
   &\Delta\mathcal{X}_k=\{\Delta x:a^\top_{x,i}\Delta x\leq b_{x,i}-a^\top_{x,i}x^{ref}_{\bar{t}+k}, \forall \in\mathbb{I}_1^{n_X} \},\nonumber\\
   &\Delta\mathcal{U}_k=\{\Delta u:a^\top_{u,i}\Delta x\leq b_{u,i}-a^\top_{u,i}u^{ref}_{\bar{t}+k}, \forall \in\mathbb{I}_1^{n_U} \}. \nonumber 
\end{align}}}

The multi-modal TV predictions \eqref{eq:TV_preds} are rather general and can also represent tree-structured predictions. For example, consider the tree on the left in Fig.~\ref{fig:mm23} where each joint TV state prediction $o_{k|t,j}$ at time step $k$ for mode $j$ represents a node, $o_{k+1|t,j}$ is the child node of $o_{k|t,j}$ and $o_{t|t}\equiv o_t$ is the root. 
The root node branches out into $J$ children, where each branch signifies a discrete decision made by the TVs to evolve in a particular mode after time $t$. In general, the branching decisions may occur at a time step further along the prediction horizon, and at each branching point, there may be fewer than $J$ children (as depicted by the tree on the right in Fig.~\ref{fig:mm23}). 
For modeling branching out from the parent node $o_{k'|t,j_1}$ at time $k'>t$ into $m$ children $o_{k'+1|t,j_1}, o_{k'+1|t,j_2},\dots, o_{k'+1|t,j_m}$, the trajectory evolution of the TVs up to time $k'$ can be constrained to match using the equalities $(T_{k|t,j_r}, c_{k|t,j_r}, n_{k|t,j_r})= (T_{k|t,j_{r+1}}, c_{k|t,j_{r+1}}, n_{k|t,j_{r+1}})$ $,\forall r\in \mathbb{I}_1^{m-1}, k\in\mathbb{I}_t^{k'-1}$, as depicted by the tree on the right in Fig.~\ref{fig:mm23}.  The equalities defining this tree structure are captured in the set $\mathcal{T}_t\subset \mathbb{I}_t^{t+N}\times\left(\mathbb{I}_1^J \times\mathbb{I}_1^J\right)$, where $(k, \{j_1, j_2\})\in\mathcal{T}_t$ corresponds to the equality $(T_{k|t,j_1},c_{k|t,j_1}, n_{k|t,j_1})=(T_{k|t,j_2},c_{k|t,j_2}, n_{k|t,j_2})$.

\begin{figure}[h]
    \centering
    \includegraphics[width=1.\columnwidth]{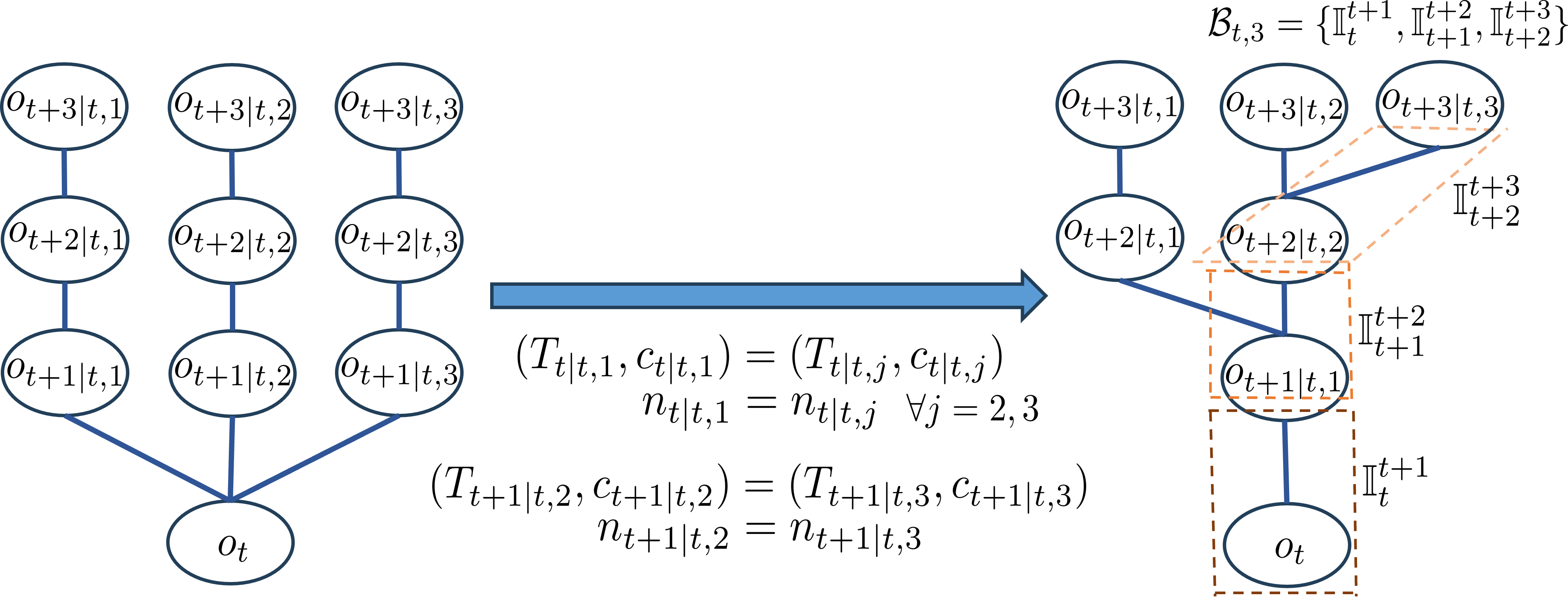}
    \caption{{\small{Modelling tree-structured multi-modal predictions. Here with $\mathcal{T}_t=\{(t,\{1,2\}), (t,\{1,3\}), (t+1,\{2,3\})\}$, we get a tree with $J=3$ leaf nodes with branch points at prediction steps $t$ and $t+1$. The branches for mode $3$ are $\mathcal{B}_{t,3}=\{\mathbb{I}_{t}^{t+1},\mathbb{I}_{t+1}^{t+2}, \mathbb{I}_{t+2}^{t+3}\}$}}}
    \label{fig:mm23}    
\end{figure}
For each mode $j$, we can define a set of \textit{branches} comprising the time intervals in the tree between branch points, defined formally as
{\small{\begin{align}\label{eq:branches}
\mathcal{B}_{t,j}=\left\{\mathbb{I}_{k_1}^{k_2} \subset\mathbb{I}_{t}^{t+N}\middle\vert \begin{aligned} \exists& j_1,j_2: (k_1,\{j,j_1\}), (k_2, \{j,j_2\})\in\mathcal{T}_t,\\
& k_2=\min(t+N,\min_{k>k_1, (k,\{j,l\})\in\mathcal{T}_t} k)
\end{aligned}\right\}
\end{align}}}
If there are no branches in mode $j$, i.e., $\mathcal{B}_{t,j}=\emptyset$, we add the trivial branch, $\mathcal{B}_{t,j}=\{\mathbb{I}_t^{t+N}\}$. 
\subsection{Parameterized EV \& TV State Feedback Policies}
\begin{figure}[!h]
    \centering
    \begin{subfigure}{0.45\columnwidth}
    \includegraphics[width=\columnwidth]{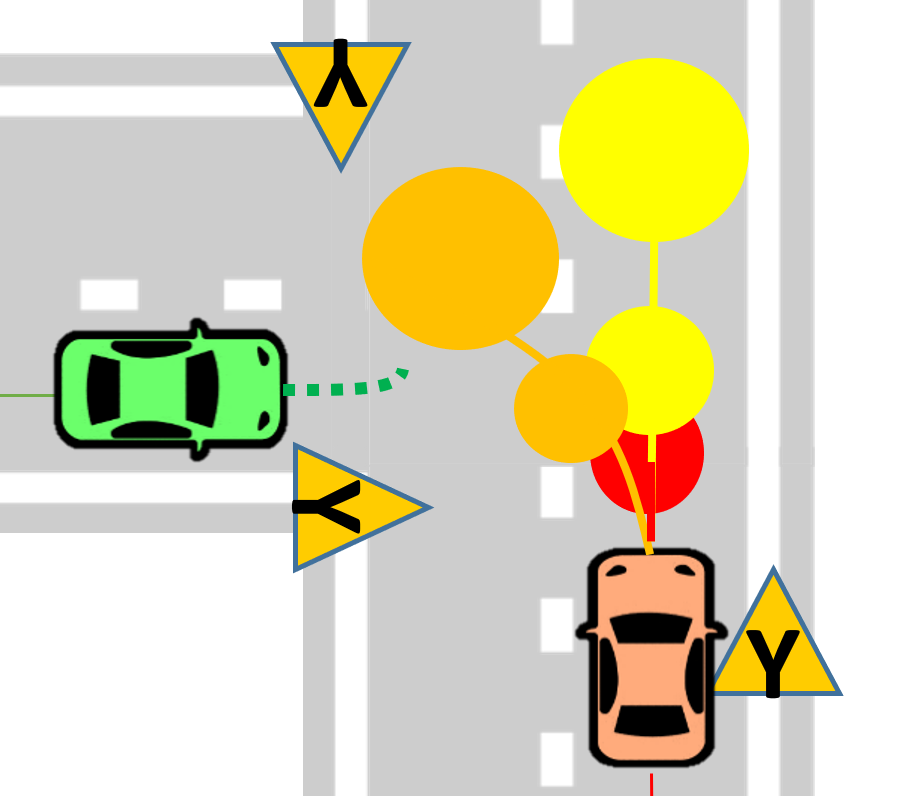} 
        \caption{Prediction with open-loop sequences $\mathbf{u}_t\in\mathbb{R}^{2\times N}$}
        \end{subfigure}\hfill %
    \begin{subfigure}{0.45\columnwidth}
        \includegraphics[width=\columnwidth]{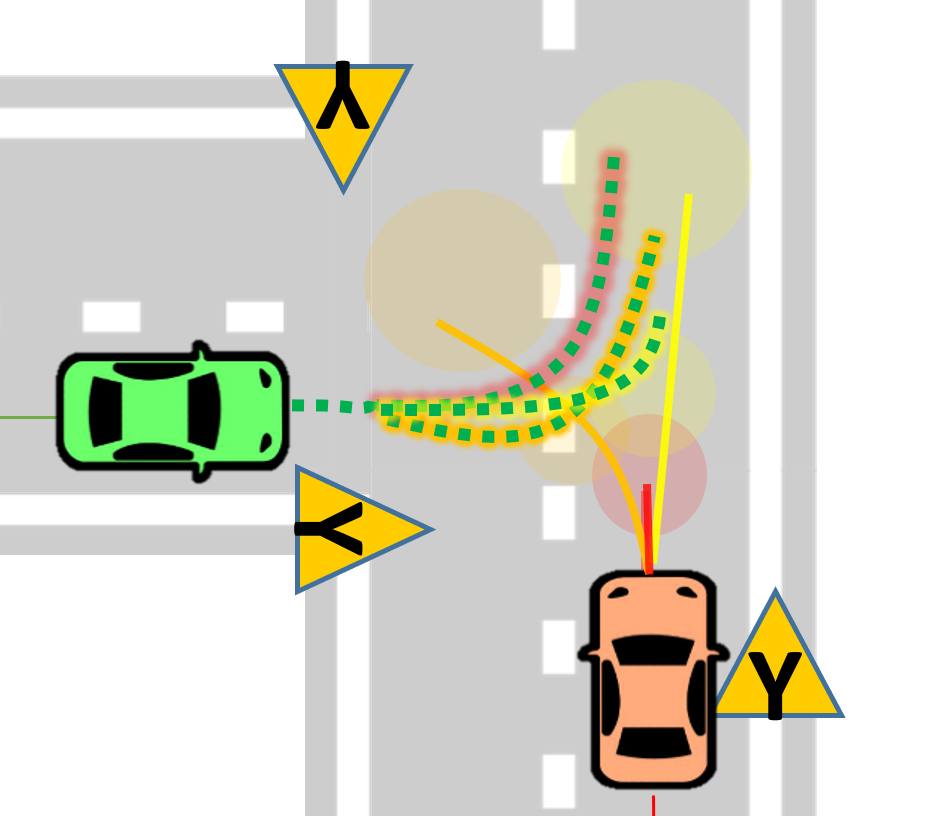} 
        \caption{Predictions with policies $\mathbf{u}_t\in\Pi_{\theta_t}(\mathbf{x}_t,\mathbf{o}_t)$}
        \end{subfigure}
    \caption{\small{In (a), solving \eqref{opt:SMPC_skeleton} over open-loop sequences can be conservative because EV prediction (green-dashed) from a single sequence of control inputs must satisfy all the obstacle avoidance constraints. In (b), optimizing over policies \eqref{opt:gen_pol_class} allows for different EV predictions depending on the TV trajectory realizations (green-dashed with highlights corresponding to different TV trajectories).\label{fig:OLvsFB}}}    
\end{figure}
We propose to use parameterized polices $\Pi_{\theta_t}(\mathbf{x}_t,\mathbf{o}_t)$ so that the EV's control $\mathbf{u}_t$ are functions of the EV and TV trajectories $\mathbf{x}_t, \mathbf{o}_t$ along the prediction horizon (as depicted in Fig. \ref{fig:OLvsFB}). Given the EV model \eqref{eq:EV_ltv_model} and mode-dependent TV model \eqref{eq:TV_preds}, consider the following feedback policy $\Delta u_{k|t,j}=\pi_{k|t,j}(x_{k|t,j}, o_{k|t,j})$ for the EV :
{\small{
\begin{align}\label{eq:policy}
    \Delta u_{k|t,j}&=h^j_{k|t}+\sum\limits_{l=0}^{k-1}M^j_{l,k|t}w_{l|t} + K^j_{k|t}(o_{k|t,j}-\mu_{k|t,j}),
\end{align}}}
where $\mu_{k|t,j}=\mathbb{E}[o_{k|t,j}]$ denotes the expected prediction of the TVs in \eqref{eq:TV_preds}. The policy \eqref{eq:policy} uses State Feedback (SF) for the TV states but Affine Disturbance Feedback (ADF) for feedback over EV states (see \cite{goulart2006optimization} for equivalence to state feedback) instead of SF. SF policies for the TVs are beneficial towards scaling our approach to multiple TVs because we use {\small{$O( N^2+n_{TV}\cdot N)$}} parameters instead of {\small{${O(n_{TV}\cdot N^2)}$}} parameters for ADF.  \\


Despite using SF for the TVs' states, $\Delta x_{k|t}$ are affine in $\{K_{s|t}\}_{s=0}^k~\forall k\in\mathbb{I}_0^{n-1}$ as shown next. For mode $j\in\mathbb{I}_1^J$, define the stacked quantities along the prediction horizon $\mathbf{\Delta u}_{t,j},\mathbf{w}_t, \mathbf{o}_{t,j}, \boldsymbol{\mu}_{t,j}$ (e.g., $\mathbf{\Delta u}_{t,j}=[\Delta u^{\top}_{0|t,j}\dots\Delta u^{\top}_{N-1|t,j}]^\top$) and use \eqref{eq:policy} to get {\small{$$\mathbf{\Delta u}_{t,j}=\mathbf{h}_t^j+\mathbf{M}_t^{j}\mathbf{w}_t+\mathbf{K}_t^j(\mathbf{o}_{t,j}-\boldsymbol{\mu}_{t,j})$$}} using the stacked policy parameters $\mathbf{h}_t^j\in\mathbb{R}^{2N}, \mathbf{M}_t^{j}\in\mathbb{R}^{2N\times 4N},  \mathbf{K}^j_{t}\in\mathbb{R}^{2N\times 2N\cdot n_{TV}}$ (\eqref{mat:h}-\eqref{mat:M} in appendix~\ref{app:matrices}). Denote the stacked EV closed-loop predictions for mode $j$ as $\boldsymbol{\Delta x}_{t,j}$ and the TV's stacked process noise as $\mathbf{n}_{t,j}$. Using matrices $\mathbf{A}_t, \mathbf{B}_t,\mathbf{T}^j_t, \mathbf{C}_t^j, \mathbf{L}_t^j, \mathbf{E}_t$  defined by \eqref{mat:AB}-\eqref{mat:EL} in the appendix, the closed-loop EV predictions are
  {\small{
  \begin{align*}
      \boldsymbol{\Delta x}_{t,j}=&\mathbf{A}_t\Delta x_{0|t}+\mathbf{B}_t(\mathbf{h}^j_t+\mathbf{K}^j_t(\mathbf{o}_{t,j}-\boldsymbol{\mu}_{t,j}))+(\mathbf{E}_t+\mathbf{B}_t\mathbf{M}^j_t)\mathbf{w}_t\\
      =&\mathbf{A}_t\Delta x_{0|t}+\mathbf{B}_t(\mathbf{h}^j_t+\mathbf{K}^j_t\mathbf{L}^j_t\mathbf{n}^j_t)+(\mathbf{E}_t+\mathbf{B}_t\mathbf{M}^j_t)\mathbf{w}_t
  \end{align*}}}
  which are affine in $\mathbf{h}^j_t,\mathbf{M}^j_t,\mathbf{K}^j_t$. \\
  
  To reflect the tree structure $\mathcal{T}_t$ of the TV predictions in the policy parameters $\boldsymbol{\Theta}_t(x_t,o_t)=\{\mathbf{h}^j_t,\mathbf{M}^j_t,\mathbf{K}^j_t\}_{j=1}^J$, we construct a new set {\small{$$\mathcal{T}_i^{\pi}=\mathcal{T}_t\cup \{(k+1,j_1,j_2)\ : (k,j_1,j_2)\in\mathcal{T}_t \},$$}}  and then impose the equalities specified in $\mathcal{T}_t^{\pi}$ on the policy parameters. $\mathcal{T}_t^{\pi}$ has a similar structure to $\mathcal{T}_t$ but the branch points occur with a single time-step delay to encode that the EV observes the mode one time-step after the branch point. For the example in Fig.~\ref{fig:mm23}, we have $\mathcal{T}_t^{\pi}=\{(t,\{1,2\}), (t+1,\{1,2\}), (t,\{1,3\}), (t+1,\{1,3\}), (t+1,\{2,3\}), (t+2,\{2,3\})\}$ and the policy parameterization for mode $j=3$ is given as
        {\small{
\begin{align*}
    &\mathbf{h}_t^3=[h^{1\top}_{t|t}\ h^{1\top}_{t+1|t}\ h^{2\top}_{t+2|t} ]^\top\\
    &\mathbf{K}_t^3=\text{blkdiag}\left(K^1_{t|t},K^1_{t+1|t},K^2_{t+2|t}\right)\\
    &\mathbf{M}_t^{3}=\begin{bmatrix}
    O& O & O \\
    M^1_{t,t+1|t}& O & O \\
    M^2_{t,t+2|t}& M^2_{t+1,t+2|t} & O 
    \end{bmatrix}
    \end{align*}}}   
Similar to $\mathcal{B}_{t,j}$, we can describe the branches per mode for the policy tree as $\mathcal{B}^\pi_{t,j}$. Then for each mode $j$, we can split the policy matrices as the sums, $\mathbf{M}^j_t=\sum_{b=1}^{|\mathcal{B}^\pi_{t,j} |}\bar{\mathbf{M}}^{j,b}_t$, $\mathbf{K}^j_t=\sum_{b=1}^{|\mathcal{B}^\pi_{t,j}|}\bar{\mathbf{K}}^{j,b}_t$, where $\bar{\mathbf{M}}^{j,b}_t, \bar{\mathbf{K}}^{j,b}_t$ have the same shape as \eqref{mat:M},\eqref{mat:K} but consist only of policy parameters corresponding to the branch $b\in\mathcal{B}^\pi_{t,j}$. We denote this policy parameterization as $\boldsymbol{\Theta}_t(x_t,o_t; \mathcal{T}_t)$.

\subsection{Collision Avoidance Formulation}
We assume that we are given or can infer the rotation matrices for the $i$th TV for each mode along the prediction horizon as 
{\small{$\{\{R^i_{k|t,j}\}_{k=1}^N\}_{j=1}^J$}}. For collision avoidance between the EV and the $i$th TV, we impose the following chance constraint
\small
\begin{align}\label{eq:oa_chance_constraint}
    \mathbb{P}(g^i_{k|t}(P_{k|t},o^i_{k|t})\geq1\ )\geq 1-\epsilon~~~\forall k\in\mathbb{I}_1^N
\end{align}\normalsize
where the EV's position $P_{k|t}=[X_{k|t}, Y_{k|t}]$ is obtained from $\mathcal{G}^\gamma(\Delta x_{k|t}+x^{ref}_{t+k})$, and 
\small\begin{align}\label{eq:g_def}
 &g^i_{k|t}(P,o)=\Big\Vert\begin{bmatrix}\frac{1}{a_{ca}}&0\\0&\frac{1}{b_{ca}}\end{bmatrix} R^i_{k|t,\sigma_t}(P-o)\Big\Vert^2.
\end{align}\normalsize 
%
$a_{ca}=a_{TV}+d_{EV}, b_{ca}=b_{TV}+d_{EV}$ are semi-axes of the ellipse containing the TV's extent with a buffer of $d_{EV}$. $g_{k|t}(P,o)\geq1$ implies that the EV's extent (modelled as a disc of radius $d_{EV}$ and centre $P$) does not intersect the TV's extent which is modelled as an ellipse with semi-axes $a_{TV}, b_{TV}$ and centre $o$, oriented by $R_{k|t,\sigma_t}$ (for using general convex sets,  \cite{nair2022collision} can be used).
This constraint is non-convex because of the integral of the nonlinear function $g^i_{k|t}(\cdot)$ over the multi-modal distribution of $(P_{k|t}, o_{k|t}, p_t)$.
To address the nonlinearity, we use the convexity of $g^i_{k|t,j}(\cdot)$ to construct its affine under-approximation $l^{i}_{k+1|t,j}(\cdot)$ by defining \small $P^i_{k|t,j}=\mu_{k|t,j}+\frac{1}{\sqrt{g^i_{k|t,j}(P^{ref}_{k|t},\mu^i_{k|t,j})}}(P^{ref}_{k|t}-\mu^i_{k|t,j})$\normalsize\ to get:
\small
\begin{align*}
l^i_{k|t,j}(P,o)=&\partial_P g^i_{k|t}(P^{i}_{k|t,j},\mu^i_{k|t,j})(P-P^{i}_{k|t,j})\\
&+\partial_o g^i_{k|t}(P^{i}_{k|t,j},\mu^i_{k|t,j})(o-\mu^i_{k|t,j}).
\end{align*}
\normalsize
The constraint {\small{$l^i_{k+1|t,j}(P_{k+1|t},o^i_{k+1|t,j})\geq 0$}}\ is an inner-approximation of {\small{$g^i_{k|t}(P_{k+1|t},o^i_{k+1|t})\geq1$}} because {\small{$g^i_{k|t}(P_{k+1|t},o^i_{k+1|t})\geq 1+l^i_{k+1|t,j}(P_{k+1|t},o^i_{k+1|t,j})$}}  by construction. The multi-modal, affine chance constraints for collision avoidance are given as 
{\small{\begin{align}\label{eq:mm_ca}
\sum_{j=1}^Jp_{t,j}\mathbb{P}\left[ l^i_{k+1|t,j}(P_{k+1|t},o^i_{k+1|t,j})\geq 0 \right]\geq 1-\epsilon.
\end{align}}}
We define the curve $\gamma(\cdot)$ using piece-wise linear segments so that $\mathcal{G}^\gamma(\Delta x_{k|t}+x^{ref}_{t+k})=\mathcal{G}^\gamma(x^{ref}_{t+k})+\partial_x\mathcal{G}^\gamma(x^{ref}_{t+k})\Delta x_{k|t}$, and the constraint \eqref{eq:mm_ca} is affine in the policy parameters $\boldsymbol{\Theta}_t(x_t,o_t; \mathcal{T}_t)$. Next, we discuss the reformulation of the multi-modal affine chance constraints \eqref{eq:EV_cc}, \eqref{eq:mm_ca}. 
\subsection{Reformulation of Multi-modal Chance Constraints}\label{ssec:mmreform}
\begin{figure}[!h]
\includegraphics[width=1.0\columnwidth]{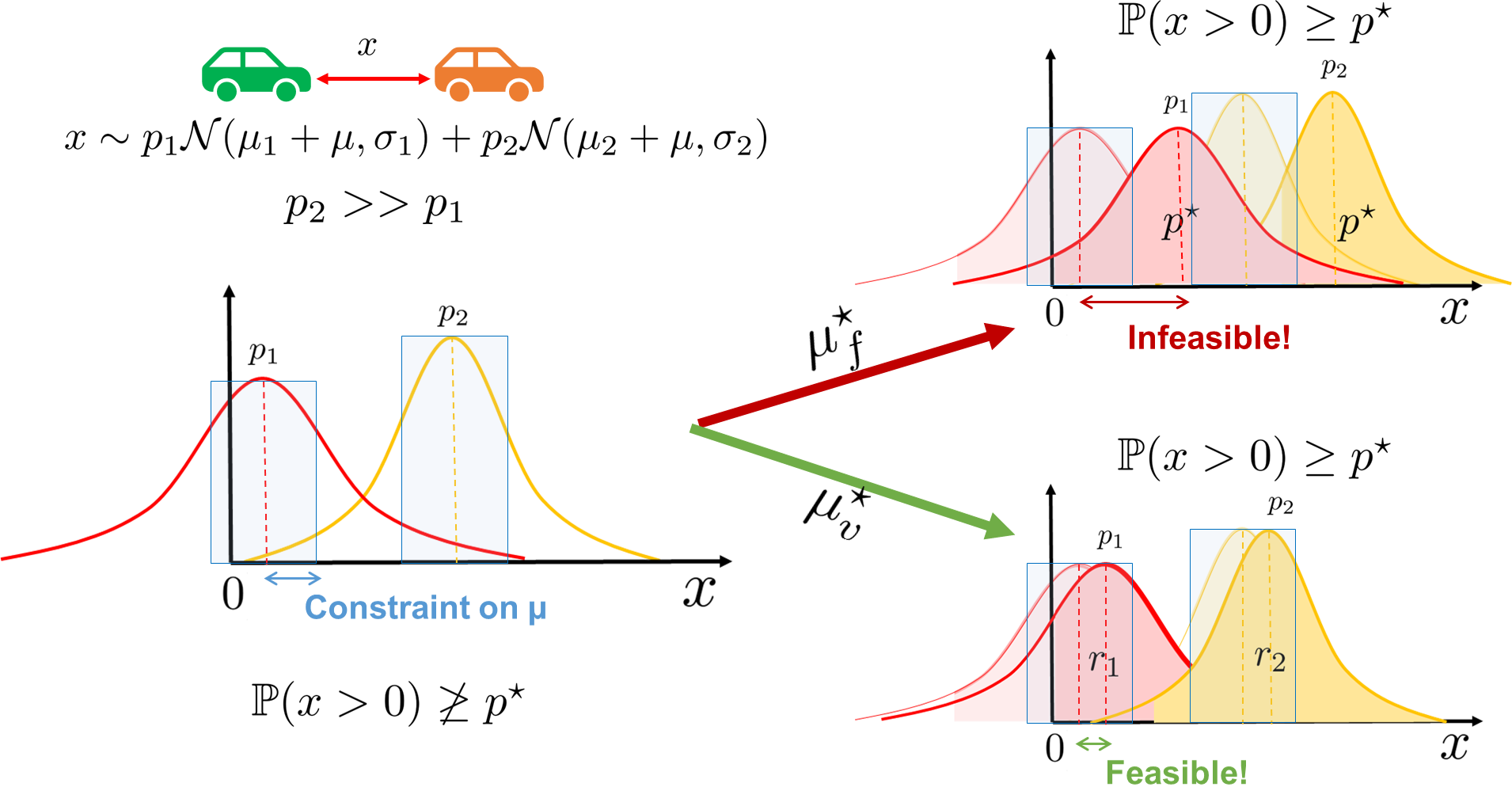}
\caption{{\small{For $x$ given by a GMM on the left, the figure depicts two formulations for shaping the distribution to satisfy $\mathbb{P}(x>0)\geq p^\star$ by varying $\mu$ within the blue shaded region. The resulting distributions are shown on the right for the fixed-risk and variable risk formulations at the top and bottom respectively. The latter allows satisfaction of the chance constraint without aggressive shaping of the distribution, by exploiting the difference in probabilities of the modes. Since $p_2\gg p_1$,  assigning a larger risk level $r_2$ (shaded yellow region) to mode 2 yields feasible distributions.}}  }\label{fig:var_risk}
\end{figure}
We propose a novel convex inner-approximation for multi-modal chance constraints, with the key feature of simultaneous risk allocation for reduced conservatism. The constraints \eqref{eq:EV_cc}, \eqref{eq:mm_ca} can be generically represented as the multi-modal affine chance constraint:
\begingroup
\allowdisplaybreaks
{\small{\begin{subequations}
\begin{align}
&\sum_{j=1}^{J}p_{t,j}\mathbb{P}\left[\begin{aligned}&a_{0}+a^{\top}_{1}\mathbf{h}^j_{t}+(a^{\top}_{2}\mathbf{M}^j_{t}+a^\top_3)\mathbf{w}_t\\&+(a^{\top}_{4}\mathbf{K}^j_t+a^{\top}_5)\mathbf{L}^j_{t}\mathbf{n}_{t}\end{aligned}\geq b\right]\geq 1-\epsilon\nonumber\\
\Leftrightarrow&\sum_{j=1}^{J}p_{t,j}r_{t,j}\geq 1-\epsilon,\label{eq:mm_psum}\\ 
&\mathbb{P}\left[\begin{aligned}&a_{0}+a^{\top}_{1}\mathbf{h}^j_{t}+(a^{\top}_{2}\mathbf{M}^j_{t}+a^\top_3)\mathbf{w}_t\\&+(a^{\top}_{4}\mathbf{K}^j_t+a^{\top}_5)\mathbf{L}^j_{t}\mathbf{n}_{t}\end{aligned}\geq b\right]\geq r_{t,j},\nonumber\\
&\Leftrightarrow \begin{aligned}a_{0}+a^{\top}_{1}\mathbf{h}^j_{t}-b\geq\Phi^{-1}(r_{t,j})\left\Vert \begin{bmatrix}\boldsymbol{\Sigma}_t(\mathbf{M}^{j\top}_{t}a_{2}+a_3)\\ \boldsymbol{\Sigma}_{t,j}\mathbf{L}^{j\top}_{t}(\mathbf{K}^{j\top}_ta_4+a_5)\end{bmatrix}\right\Vert_2\end{aligned}\label{eq:mm_affine_cc}
\end{align}
\end{subequations}}}
\endgroup where $\Phi^{-1}(\cdot)$ is the quantile function for $\mathcal{N}(0,1)$, and $r_{t,j}$ are the risk-levels $\forall j\in\mathbb{I}_1^J$. We assume that  $1-\epsilon \geq \frac{1}{2}$ and $r_{t,j}\geq \frac{1}{2},$ $\forall j\in\mathbb{I}_1^J$. As such, this constraint is non-convex in the policy parameters $\boldsymbol{\Theta}_t(x_t,o_t; \mathcal{T}_t)$ and risk-levels $r_{t,j}$ because of \eqref{eq:mm_affine_cc}. The fixed-risk allocation approach \cite{zhou2018joint, nair2022stochastic} fixes the risk-level $r_{t,j}=1-\epsilon$ $\forall j\in\mathbb{I}_1^J$ to obtain a convex-inner approximation of the multi-modal chance constraint. However as depicted in Fig.~\ref{fig:var_risk}, this approach can be conservative and compromises the feasibility of the SMPC optimization problem because the risk levels are allocated disregarding the probability of the individual modes.
Alternatively, iterative solution strategies \cite{ono2008iterative}, have been proposed for variable risk allocation where, alternating sub-problems are solved by fixing either the policy parameters or risk levels. This enhances the feasibility of the optimization problem but at the price of significant computational cost. Next, we propose a convex-inner approximation to the multi-modal chance constraint \eqref{eq:mm_psum}, \eqref{eq:mm_affine_cc} for simultaneous risk-allocation and policy synthesis, to alleviate the computational cost of iterative approaches, but also improve the feasibility of constraint compared to the fixed allocation approach. \\

First, we focus on reformulating \eqref{eq:mm_affine_cc}.  We recall that the policy parameters can be split as $\mathbf{M}^j_t=\sum_{b=1}^{|\mathcal{B}^\pi_{t,j} |}\bar{\mathbf{M}}^{j,b}_t$, $\mathbf{K}^j_t=\sum_{b=1}^{|\mathcal{B}^\pi_{t,j}|}\bar{\mathbf{K}}^{j,b}_t$, and  begin by introducing new variables 
{\small{\begin{align*}
&\eta_{t,j}=\Phi^{-1}(r_{t,j}),\\
&\tilde{\mathbf{M}}^j_t=\Phi^{-1}(r_{t,j})\bar{\mathbf{M}}^{j,|\mathcal{B}^\pi_{t,j}|}_t,\ 
\tilde{\mathbf{K}}^j_t=\Phi^{-1}(r_{t,j})\bar{\mathbf{K}}^{j,|\mathcal{B}^\pi_{t,j}|}_t \end{align*} }}

to rewrite the  constraint \eqref{eq:mm_affine_cc} as
\small
\begin{align}\label{eq:noncvx_tight}
&a_{0}+a^{\top}_{1}\mathbf{h}^j_{t}-b\geq\nonumber\\
&\left\Vert\begin{bmatrix}\boldsymbol{\Sigma}_t(\tilde{\mathbf{M}}^{j\top}_{t}a_{2}+\eta_{t,j}\sum_{b=1}^{|\mathcal{B}^\pi_{t,j} |-1}\bar{\mathbf{M}}^{j,b\top}_ta_{2}+a_3\eta_{t,j})\\ \boldsymbol{\Sigma}_{t,j}\mathbf{L}^{j\top}_{t}(\tilde{\mathbf{K}}^{j\top}_ta_4+\eta_{t,j}\sum_{b=1}^{|\mathcal{B}^\pi_{t,j} |-1}\bar{\mathbf{K}}^{j,b\top}_ta_4+a_5\eta_{t,j})\end{bmatrix}\right\Vert_2
\end{align}
\normalsize
The variable $\eta_{t,j}$ can be interpreted as the number of standard deviations by which the affine constraint is tightened.
Since $r_{t,j}\geq \frac{1}{2}$, we have $\eta_{t,j}\geq0$, and additionally, let $\eta_{t,j}\leq \eta_{max}$ to ignore the tail of the Gaussian distribution. The inequality \eqref{eq:noncvx_tight} is non-convex in the new variables because of the bilinear terms $\eta_{t,j}\bar{\mathbf{M}}^{j,b\top}_{t}$, $\eta_{t,j}\bar{\mathbf{K}}^{j,b\top}_{t}$. However, fixing either variable in these terms yields a convex, second-order cone constraint. We use this insight to obtain a convex inner-approximation of this non-convex inequality as follows.\\

\begin{prop}A convex inner approximation of \eqref{eq:noncvx_tight} can be obtained as the intersection of the two second-order cone constraints in the variables $\{\mathbf{h}^j_t, \tilde{\mathbf{M}}^j_t, \tilde{\mathbf{K}}^j_t, \{\bar{\mathbf{M}}^{j,b}_t, \bar{\mathbf{K}}^{j,b}_t\}_{b=1}^{|\mathcal{B}^\pi_{t,j}|-1}, \eta_{t,j} \}$:
\small
\begin{align}\label{eq:mm_cvx_tight}
&a_{0}+a^{\top}_{1}\mathbf{h}^j_{t}-b\geq\left\Vert \begin{bmatrix}\boldsymbol{\Sigma}_t(\tilde{\mathbf{M}}^{j\top}_{t}a_{2}+a_3\eta_{t,j})\\ \boldsymbol{\Sigma}_{t,j}\mathbf{L}^{j\top}_{t}(\tilde{\mathbf{K}}^{j\top}_ta_4+a_5\eta_{t,j})\end{bmatrix}\right\Vert_2, \nonumber \\
&a_{0}+a^{\top}_{1}\mathbf{h}^j_{t}-b\geq\nonumber\\
&\left\Vert \begin{bmatrix}\boldsymbol{\Sigma}_t(\tilde{\mathbf{M}}^{j\top}_{t}a_{2}+\eta_{max}\sum_{b=1}^{|\mathcal{B}^\pi_{t,j} |-1}\bar{\mathbf{M}}^{j,b\top}_ta_{2}+a_3\eta_{t,j})\\ \boldsymbol{\Sigma}_{t,j}\mathbf{L}^{j\top}_{t}(\tilde{\mathbf{K}}^{j\top}_ta_4+\eta_{max}\sum_{b=1}^{|\mathcal{B}^\pi_{t,j} |-1}\bar{\mathbf{K}}^{j,b\top}_ta_4+a_5\eta_{t,j})\end{bmatrix}\right\Vert_2
\end{align}
\normalsize
\end{prop}
\begin{proof}
First, we note the following auxilliary result: For a convex function $f(\cdot)$ and $x\in [x_{min},x_{max}]\subset \mathbb{R}$, if $f(x_{min}y)\leq 0$, $f(x_{max}y)\leq 0$, then $f(xy)\leq 0$ for any $x\in [x_{min},x_{max}]$. 
Now suppose that \eqref{eq:mm_cvx_tight} holds, and use the above result for an  arbitrary $\tilde{\eta}_{t,j}=\gamma \eta_{max}+(1-\gamma)0$ with $\gamma\in[0,1]$ to get:
\small
\begin{align*}
&a_{0}+a^{\top}_{1}\mathbf{h}^j_{t}-b\geq\\
&\left\Vert \begin{bmatrix}\boldsymbol{\Sigma}_t(\tilde{\mathbf{M}}^{j\top}_{t}a_{2}+\tilde{\eta}_{t,j}\sum_{b=1}^{|\mathcal{B}^\pi_{t,j} |-1}\bar{\mathbf{M}}^{j,b\top}_ta_{2}+a_3\eta_{t,j})\\ \boldsymbol{\Sigma}_{t,j}\mathbf{L}^{j\top}_{t}(\tilde{\mathbf{K}}^{j\top}_ta_4+\tilde{\eta}_{t,j}\sum_{b=1}^{|\mathcal{B}^\pi_{t,j} |-1}\bar{\mathbf{K}}^{j,b\top}_ta_4+a_5\eta_{t,j})\end{bmatrix}\right\Vert_2
\end{align*}
\normalsize
Since, $\tilde{\eta}_{t,j}$ is arbitrary, the desired result is obtained by setting $\tilde{\eta}_{t,j}=\eta_{t,j}$.
\end{proof}

With the new variable definitions, constraint \eqref{eq:mm_psum} takes the form $\sum_{j=1}^{J}p_{t,j}\Phi(\eta_{t,j})\geq 1-\epsilon$, which is convex in $\eta_{t,j}$, but difficult to enforce since $\Phi(\cdot)$ lacks a closed-form expression. For $\eta\in[0,\eta_{max}]$, we approximate $\Phi(\eta)$ by a concave function $\Psi(\eta)=\min_{i=1,..,\nu}\{ q^1_i\eta+q^0_i\}$ such that $\Phi(\eta)\geq \Psi(\eta)$, and replace \eqref{eq:mm_psum} with the convex inner-approximation:
\small
\begin{align}\label{eq:mm_psum_cvx}
&\sum_{j=1}^{J}p_{t,j}\Psi(\eta_{t,j})\geq 1-\epsilon
\end{align}
\normalsize
A candidate approximation of $\Phi(\cdot)$ over $[0,2]$ with $\nu=2$ affine functions is shown in Fig.~\ref{fig:Psi}
\begin{figure}[!h]
    \centering
    \includegraphics[width=.8\columnwidth]{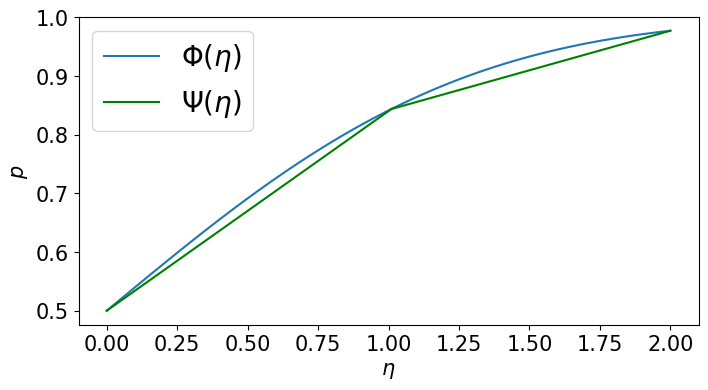}
    \caption{{\small{$\Psi(\eta)$ as a concave under-approximation of $\Phi(\eta)$}}}
    \label{fig:Psi}    
\end{figure}
Thus, the non-convex constraints \eqref{eq:mm_psum}, \eqref{eq:mm_affine_cc} can be replaced by the convex inner-approximations provided by  the Second-Order Cone (SOC) constraints \eqref{eq:mm_psum_cvx}, \eqref{eq:mm_cvx_tight}.
\subsection{SMPC Optimization Problem}
The cost \eqref{opt:obj} of the SMPC optimization problem \eqref{opt:SMPC_skeleton} can be chosen to penalise deviations of the EV state and input trajectories from the reference trajectory \eqref{eq:EV_ref_traj} as
\small
\begin{align*}
    C_t(\mathbf{x}_t,\mathbf{u}_t)=\sum\limits_{k=0}^{N-1}\Delta x_{k+1|t}^\top Q\Delta x_{k+1|t} + \Delta u_{k|t}^\top R\Delta u_{k|t}
\end{align*}
\normalsize
where $Q,R\succ 0$. Let $\boldsymbol{\Theta}^{\text{MPC}}_t(x_t,o_t; \mathcal{T}_t)$ denote the set of policy parameters and risk levels,  $$\boldsymbol{\theta}_t:=\{\mathbf{h}^j_t, \tilde{\mathbf{M}}^j_t, \tilde{\mathbf{K}}^j_t, \{\bar{\mathbf{M}}^{j,b}_t, \bar{\mathbf{K}}^{j,b}_t\}_{b=1}^{|\mathcal{B}^\pi_{t,j}|-1}, \eta_{t,j} \}_{j=1}^J$$ such that they satisfy \begin{enumerate}
    \item the reformulations \eqref{eq:mm_cvx_tight}, \eqref{eq:mm_psum_cvx} of the multi-modal state-input and collision avoidance constraints \eqref{eq:EV_cc}, \eqref{eq:mm_ca},
    \item structural constraints given by $\mathcal{T}^\pi_t$.
\end{enumerate}
Thus, the SMPC \eqref{eq:SMPC} for the EV can be synthesized by solving the SOCP:
{\small{\begin{align}\label{opt:SMPC}
\min_{\boldsymbol{\theta}_t}~~&\sum_{j=1}^J p_{t,j}\mathbb{E}\left[C_t(\mathbf{x}_{t,j},\mathbf{u}_{t,j})\right]\nonumber\\
\text{s.t}&~\boldsymbol{\theta}_t\in \boldsymbol{\Theta}^{\text{MPC}}_t(x_t,o_t;\mathcal{T}_t)
\end{align}}}
%

\section{Numerical Validation}
\label{sec:sim_implement}
In this section, we investigate the proposed algorithm, both qualitatively and quantitatively.
To assess the benefits of our proposed stochastic MPC formulation, we demonstrate our approach in three different scenarios: (A) \textit{Longitudinal control with a traffic light and a following vehicle}, (B) \textit{Unprotected left turn at an intersection}, and (C) \textit{Lane change on a straight road}. 
In scenario (A), we validate the proposed algorithm in a simple 1-D simulation and show the qualitative behavior of our SMPC in managing the multi-modal predictions. 
In scenarios (B) and (C), we use CARLA \cite{carla_sim_2017} for the simulator and adopt the motion predictor MultiPath \cite{multipath_2019} to predict the multi-modal future motions of surrounding vehicles. In these scenarios, we provide a quantitative study of our approach against baselines.
\subsection{Qualitative analysis: Longitudinal control with a traffic light and a following vehicle} \label{sec: validation A}
\subsubsection*{Setup}
Consider the situation in Fig.\ref{fig:TL}, where the EV is approaching a Traffic Light (TL) with a tailgating TV behind. All vehicles and the traffic light are simulated in a simple 1D simulator. Both vehicles are modelled as double integrators with Euler discretization (@$dt=0.1s$) as follows:

\begin{equation}
    {\small{\begin{aligned}
        & x_t = \begin{bmatrix} s_t & v_t \end{bmatrix}^\top, ~ u_t=a_t, \\
        & x_{t+1} = \begin{bmatrix} 1 & dt \\ 0 & 1 \end{bmatrix} x_t + \begin{bmatrix} 0.5 dt^2 \\ dt \end{bmatrix} u_t, \\
        & o_t = \begin{bmatrix} s^o_t & v^o_t \end{bmatrix}^\top, ~ u^o_t=a^o_t,  \\
        & o_{t+1} = \begin{bmatrix} 1 & dt \\ 0 & 1 \end{bmatrix} o_t + \begin{bmatrix} 0.5 dt^2 \\ dt \end{bmatrix} u^o_t + w^o_t, \\
        & w^o_t \sim \mathcal{N}(0, 0.6 I_2),
    \end{aligned}}}
\end{equation}
where each state comprises the longitudinal position and speed, each input is the acceleration, and $w^o_t$ is an additive process noise in the TV dynamics.
For our simulation, the initial states of the EV and the TV are set to $x_0=[0,\ 13.9],~ o_0=[-12.75, 14]$ 
so that 1) TV has a $0.7s$ time headway behind the EV, 2) there is enough distance for the EV to brake at $0.7g$ and stop at the TL, which is located $50m$ ahead of the EV, i.e., $s_f=50m$. 

The EV is subject to state-input constraints $\mathcal{X}\times\mathcal{U}=\{(x,u): v\in[v_{min}, v_{max}], a\in[a_{min},a_{max}]\}$ and collision avoidance constraints $\mathcal{C}=\{(x,o): s-o\geq d_{safe}\}$,  
where $v_{min}=0m/s,\ v_{max}=14m/s,\ a_{min}=-7m/s^{2},\ a_{max}=4m/s^2,\ d_{safe}=7m$.

For simplicity, we assume that the driver in the TV has good decision-making skills based on the previous observations. In particular, we assume the following: 
\begin{itemize}
    \item When the TV's driver is confident that the TL will remain yellow until the TV crosses with a probability of 1, the TV will maintain its speed.
    \item When the TV's driver is not confident that the TL will remain yellow, they will choose to brake. Here, we assume the probability of either red or yellow light is 0.5 conditioned on the TV choosing to brake.
    \item The TV's driver will make a decision when the TV passes a certain point, i.e., $s^o_t \geq s_{dec}$. 
\end{itemize}
Under this assumption, there are three possible modes:
\begin{itemize}
    \item mode $0$: TV keeps speed, TL stays yellow,
    \item mode $1$: TV brakes, TL stays yellow,
    \item mode $2$: TV brakes, TL goes red before EV crosses.
\end{itemize}
The EV does not know the true mode $\sigma \in \{0,1,2\}$. The EV estimates the probabilities of each mode as $p_t=[p_{t,0}, p_{t,1}, p_{t,2}]$ using Bayes' rule via observations of the TV's state history $o_{0},o_{1},..o_{t}$ after crossing $s_{dec}$. The tree structure $\mathcal{T}_t$ for the predictions is determined by rolling out the TV's acceleration commands and branching into the three modes based on when it crosses $s_{dec}$.

The chance constraints are imposed with risk level $\epsilon=0.01$. For deriving the variable-risk reformulations of multi-modal chance constraints \eqref{eq:EV_cc},\eqref{eq:mm_ca} as shown in Section \ref{ssec:mmreform}, we use $\Psi(\eta)$ as depicted in Fig.~\ref{fig:Psi} as the concave under-approximation of the CDF $\Phi(\eta)$ over $\eta\in[0,2]$. For stopping at $s_{f}$ in $mode\ 2$, we enforce the terminal constraint $\mathcal{X}_f=\{x=[s,v]\in\mathcal{X}: v^2\leq -2a_{min}(s_f-s) \}.$

The SMPC cost is given as $C_t(\mathbf{x}_{t,j},\mathbf{u}_{t,j})=\sum_{k=t}^{t+N} -Qs_{k+1|t,j}+ Ru^2_{k|t,j}$ over prediction horizon $N=12$ with $Q=10, R=20$ to penalize slow progress and control effort. 

\subsubsection*{Simulations}
We run two different simulation scenarios, where true modes are $\sigma=0$ and $\sigma=2$, respectively. Note that the EV does not know the true mode and estimates the mode after TV crosses the decision point $s_t^o \geq s_{dec}$.
We compare the proposed SMPC in \eqref{opt:SMPC} (\textit{Proposed} in Fig.~\ref{fig:TL_plot}) with an open-loop approach (\textit{OL} in Fig.~\ref{fig:TL_plot}), where the EV solves an SMPC with fixed-risk levels for each mode, and optimizes over a single open-loop sequence $\mathbf{h}_t$ i.e., the gains $K^j_{k|t}, M^j_{l,k|t}$ in \eqref{eq:policy} are eliminated. The results are illustrated in Fig.~\ref{fig:TL_plot}.

\begin{figure}[!h]
    \centering
    \includegraphics[width=1.\columnwidth]{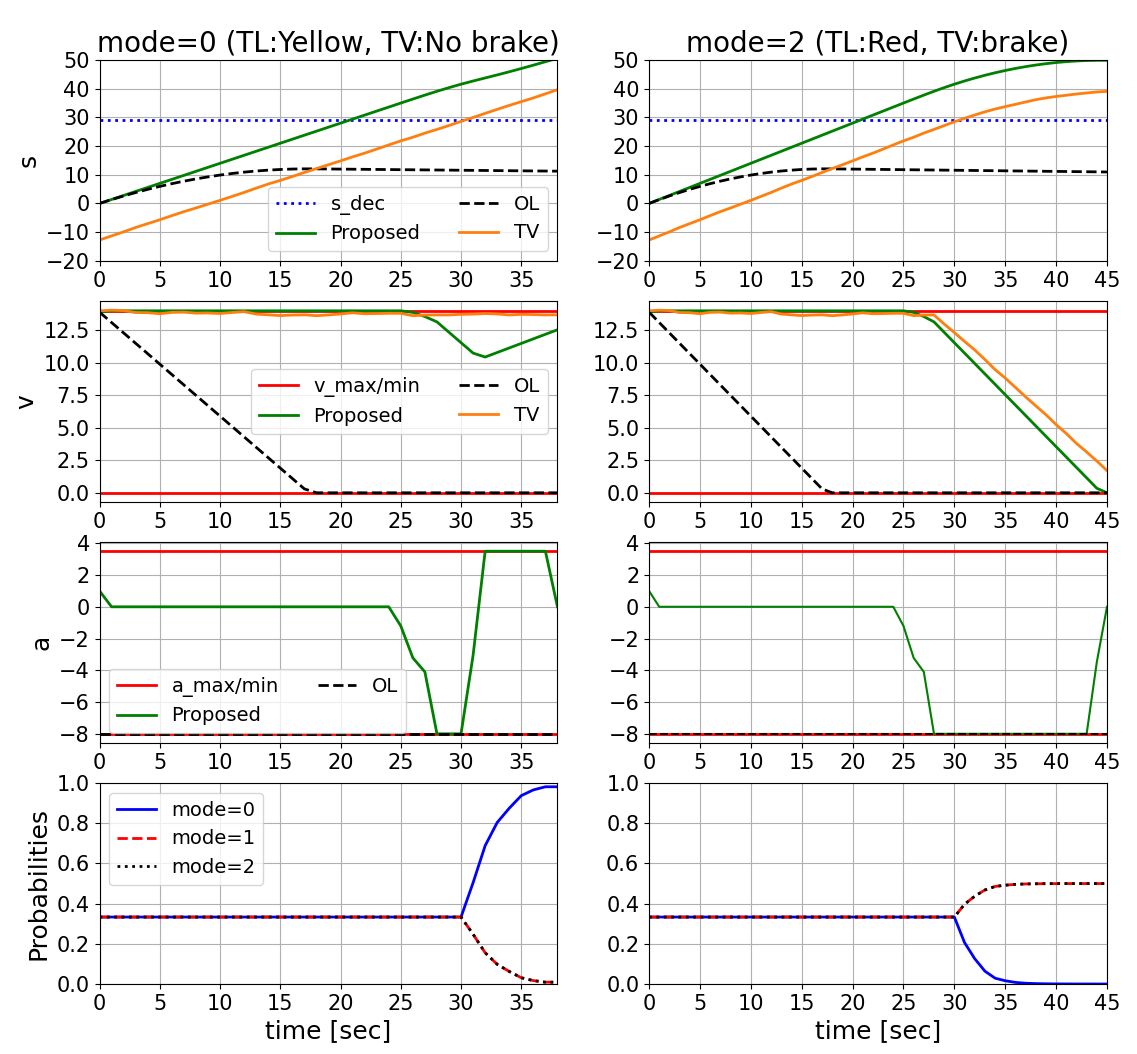}
    \caption{{\small{Closed-loop plots for the longitudinal control example for modes $\sigma=0,2$ by solving \eqref{opt:SMPC} for the SMPC. \textit{Proposed} is our proposed approach and \textit{OL} is the open-loop approach. The traffic light (TL) is located at $s_f = 50m$. The EV is unaware of the true mode $\sigma$, but  estimates the mode probabilities from TV observations. Using the proposed approach, the EV is able to cross the TL for $\sigma=0$ and safely stop before the TL for $\sigma=2$. Code: {\footnotesize{\url{https://github.com/shn66/AV_SMPC_Demos/tree/TL_eg}}}.}}}
    \label{fig:TL_plot}    
\end{figure}

\subsubsection*{Discussion} 
As depicted in Fig.~\ref{fig:TL_plot}, the open-loop approach is unable to exploit the mode probabilities and is infeasible throughout the simulation, leading to a collision with the TV.
In contrast, the proposed approach accordingly accelerates to cross the yellow TL in the first realization, while it decelerates to a stop in the second realization (without knowing the true mode $\sigma$) after the TV crosses the decision point $s_t^o \geq s_{dec}$ (around $t=30$ sec).
This closed-loop behavior results from the estimation of the mode probabilities and the incorporation of the multi-modal probability estimates in the chance constraints as in \eqref{eq:mm_psum_cvx}.
\subsection{Quantitative analysis: Unprotected left turn and Lane change with surrounding vehicles}\label{ssec:sim_2}

\begin{figure}[!h]
    \centering
    \begin{subfigure}{0.8\columnwidth}
    \includegraphics[width=\columnwidth]{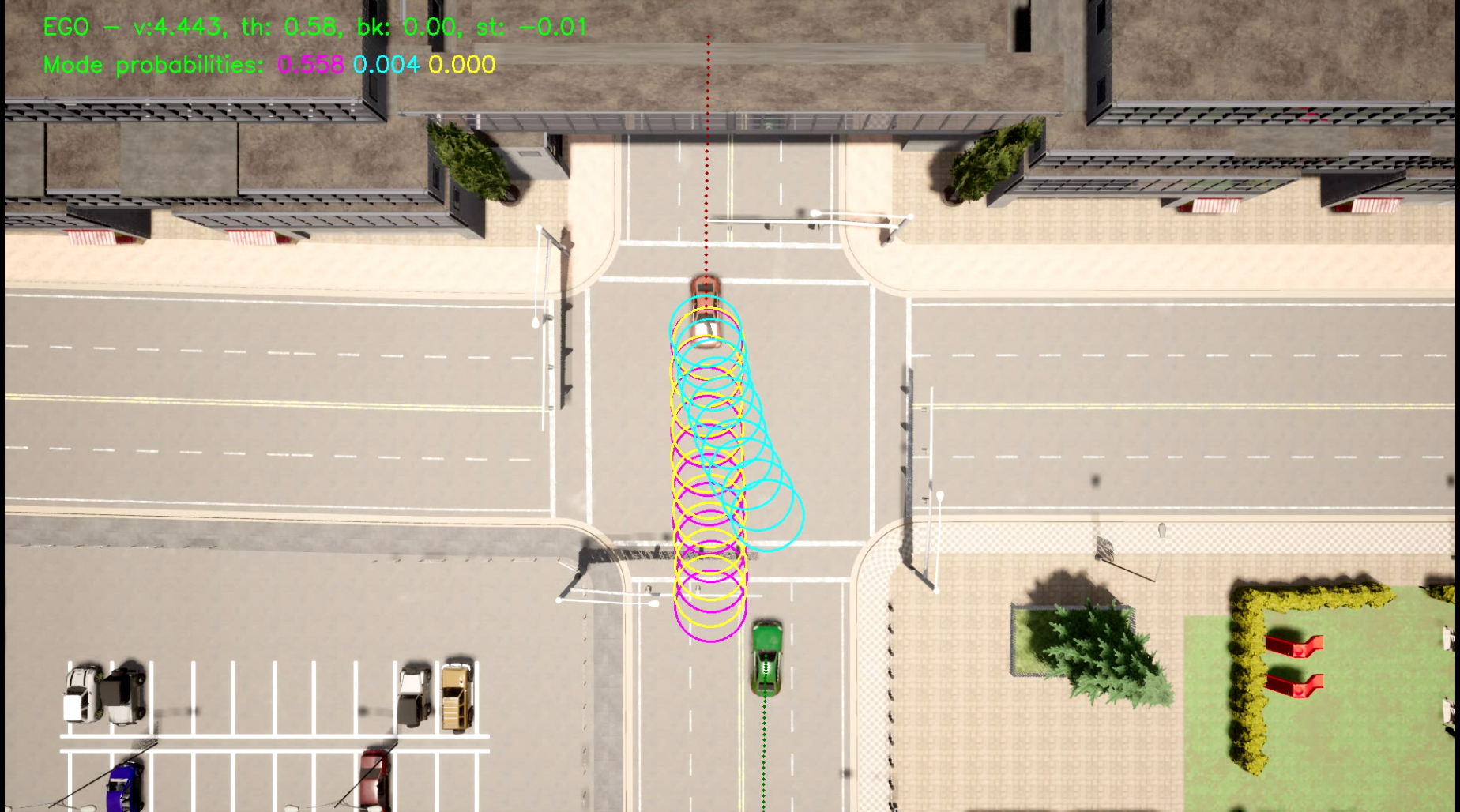} 
        \caption{Unprotected left}
        \end{subfigure}\hfill %
    \begin{subfigure}{0.8\columnwidth}
        \includegraphics[width=\columnwidth]{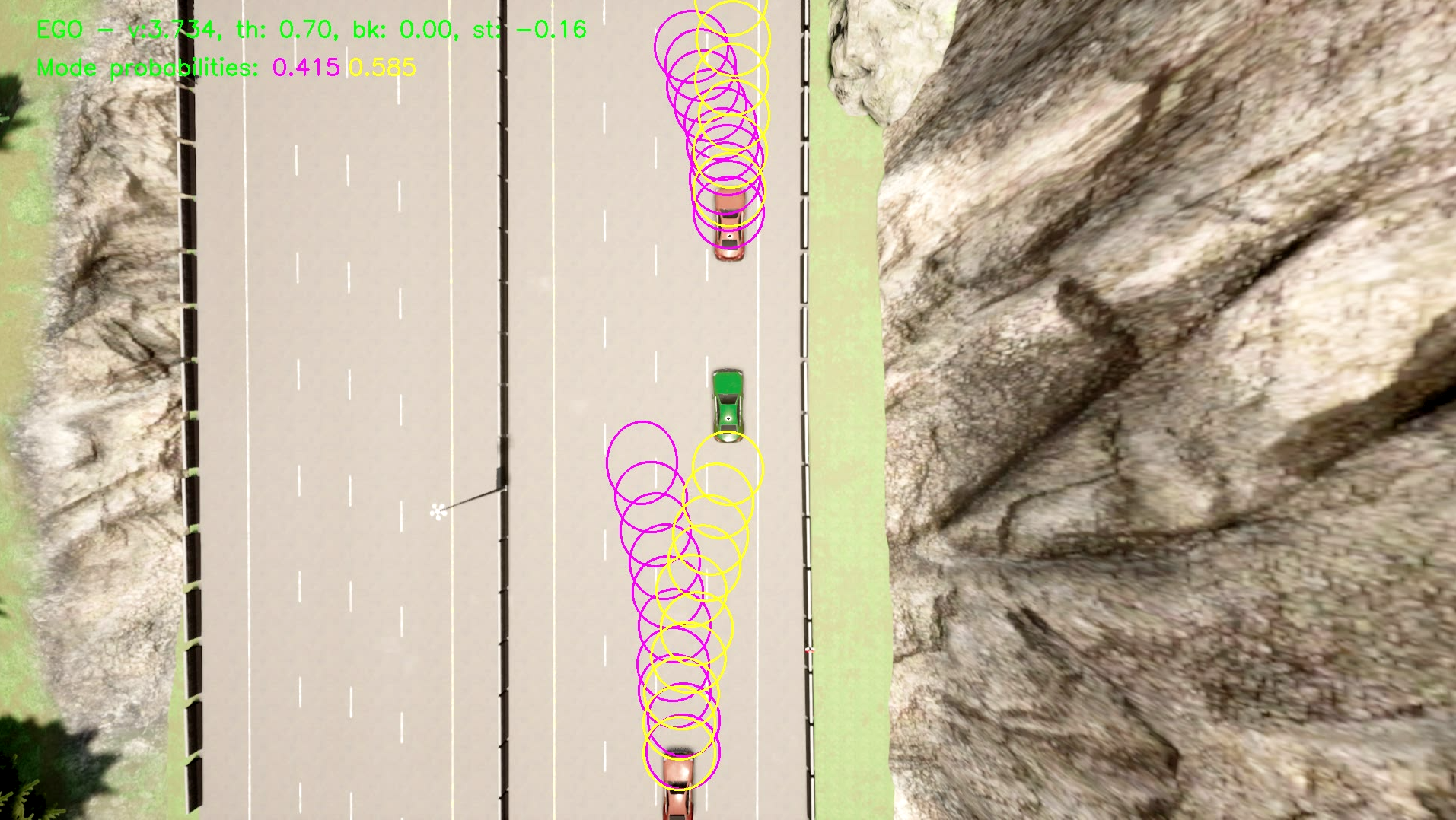} 
        \caption{Lane change}
        \end{subfigure}
    \caption{{\small{Carla simulation setup for unprotected left turn and lane change for EV (green) in the presence of TVs (orange) and multi-modal predictions (depicted by the ellipses). Code: \url{https://github.com/shn66/SMPC_MMPreds}  }}}\label{fig:scen_carla}    
\end{figure}
\begin{table*}[t!]
    \centering
    \caption{Closed-loop performance comparison across all scenarios. }  
    \label{tab:comparison_scenario_all}    
    \begin{tabular}[c]{|c |c | *{2}{c} | *{3}{c} | *{2}{c} | *{1}{c} |}
        \hline
        \multirow{2}{*}{\textbf{Scenario}} &
        \multirow{2}{*}{\textbf{Policy}} &
        \multicolumn{2}{c|}{\textbf{Mobility}} &
        \multicolumn{3}{c|}{\textbf{Comfort}} &
        \multicolumn{2}{c|}{\textbf{Safety}} &
        \multicolumn{1}{c|}{\textbf{Solver Performance}} \\
        & & \mobility & \comfort & \conservatism & \efficiency \\
        \hline
         \multirow{3}{*}{Unprotected left} & \textbf{OL}    & 1.53 & 1.96 & 1.31 & 8.09 & $\mathbf{6.62}$  & 81.14 & 3.88 &   $\mathbf{13.6}$ \\
        &\textbf{Fixed risk}        & 1.10 & 3.09 & 1.41 & 4.23 &  9.04  & 98.37 & 3.09 &  31.5 \\
        &\textbf{Proposed}        & $\mathbf{1.09}$ & $\mathbf{3.07}$ & $\mathbf{1.21}$ & $\mathbf{3.67}$ &  8.58 & $\mathbf{99.88}$ & 3.07  &  39.9 \\
        \hline
        \multirow{3}{*}{Lane change} & \textbf{OL}    & 1.32 & 7.71 & 8.16 & 4.41 & 6.45  & 82.55 & 3.42 &   $\mathbf{35.6}$ \\
        & \textbf{Fixed risk}        & 1.07 & 3.03 & 5.71 & 3.80 &  4.93  & 96.08 & 3.21 &  325.20 \\
        & \textbf{Proposed}        & $\mathbf{1.06}$ & $\mathbf{2.92}$ & $\mathbf{4.17}$ & $\mathbf{3.25}$ &  $\mathbf{3.27}$ & $\mathbf{98.76}$ & 3.19  &  397.19 \\
        \hline
    \end{tabular}
\end{table*}
\subsubsection*{Setup}
Consider the two scenarios as depicted in Fig.~\ref{fig:scen_carla}:
\begin{itemize}
    \item \textit{Unprotected left}: the EV makes a left turn through the intersection while avoiding an oncoming TV.
    \item \textit{Lane change}: the EV changes into the left lane in the presence of two TVs: One ahead in the same lane and another behind in the adjacent lane.
\end{itemize} 
All vehicles are simulated in a synchronous fashion in CARLA, ensuring that all processing (prediction, planning, and control) is complete before advancing the simulation.  Thus, our results only consider the impact of the control and not the time delays incurred in computation.

The TVs' controls are given by a simple nonlinear MPC to go straight, with a kinematic bicycle model for predictions and simple distance-based collision avoidance constraints. From the EV's perspective, the predictions of the TV's motion are given by a GMM for multi-modal trajectory predictions of the form \eqref{eq:TV_preds}. We obtain this multi-modal distribution along with mode probabilities online using the MultiPath prediction model from \cite{multipath_2019}. Given the predicted multimodal distributions for the TVs, we use the framework introduced in Section~\ref{sec:SMPC} to generate feedback control policies, and use Fig.~\ref{fig:Psi} for the CDF approximation. A dynamically feasible EV reference trajectory \eqref{eq:EV_ref_traj} is obtained by solving a nonlinear trajectory optimization problem for a kinematic bicycle using IPOPT~\cite{ipopt_2006} to track a high-level route (provided by the CARLA waypoint API). Given the EV reference \eqref{eq:EV_ref_traj} and TV predictions \eqref{eq:TV_preds}, the SMPC optimization problem \eqref{opt:SMPC} is solved using Gurobi \cite{gurobi} to compute the acceleration and steering controls.

\subsubsection*{Policies} 
We evaluate and compare the following set of policies for this unprotected left scenario: 
\begin{itemize}
    \item \textbf{Proposed}: Our proposed framework, given by solving \eqref{opt:SMPC} which optimizes over both, policies and risk levels for the multi-modal chance constraint \eqref{eq:mm_psum}, \eqref{eq:mm_affine_cc}.
    \item \textbf{Fixed risk}: An ablation of our approach, which optimizes over policies but with fixed risk levels $r_{t,j}=1-\epsilon$ for the multi-modal chance constraint \eqref{eq:mm_psum}, \eqref{eq:mm_affine_cc}.
    \item \textbf{OL}: An ablation of our approach, where the gains $K^j_{k|t}, M^j_{l,k|t}$ are eliminated and risk levels are fixed too.
\end{itemize}
Note that \textbf{Proposed} and \textbf{OL} in this section are the same algorithms that were compared in Sec. \ref{sec: validation A}.

\subsubsection*{Evaluation Metrics}
We introduce a set of closed-loop behavior metrics to evaluate the policies.  A desirable planning framework enables high mobility without being over-conservative, allowing the timely completion of the driving task while maintaining passenger comfort. The computation time should also not be exorbitant to allow for real-time processing of updated scene information.
The following metrics are used to assess these factors:
\begin{itemize} 
    \item \textbf{Mobility}: 1) $\tilde{\mathcal{T}}_{episode}$: Time the EV takes to reach its goal normalized by the time taken by the reference, and 2) $\Delta\tau$: Deviation (in Hausdorff distance) of the closed-loop trajectory from the reference trajectory, 
    A large deviation $\Delta\tau$ indicates a longer trajectory, leading to an extended time required to complete the maneuver. 
    \item \textbf{Comfort}: 1) $\tilde{\mathcal{A}}_{lat}$: Peak lateral acceleration normalised by the peak lateral acceleration in reference, 2) $\bar{\mathcal{J}}_{long}$: Average longitudinal jerk,  and 3) $\bar{\mathcal{J}}_{lat}$: Average lateral jerk.  High values are undesirable, linked to sudden braking or steering.
    \item \textbf{Safety}: 1) $\mathcal{F}$: Feasibility \% of the SMPC optimization problem. A high $\mathcal{F}$ value is desirable, as infeasibility of the SMPC can potentially lead to accidents. 
    2) $\bar{d}_{min}$: Closest distance between the EV and TV of each algorithm, provided that the algorithm remains feasible. A higher $\bar{d}_{min}$ indicates that the algorithm should be more conservative to maintain safety. This caution can lead to reduced feasibility when the algorithm encounters congested urban road driving scenarios.
    \item \textbf{Solver Performance}: 1) $\bar{\mathcal{T}}_{solve}$: Average time taken by the solver; lower is better.
\end{itemize}

\subsubsection*{Discussion}

Now, we present the results of the various SMPC policies. For each scenario, we roll out each policy for 50 different initial conditions by varying: 1) starting positions within $[-5 \text{m}, 5\text{ m}]$ and 2) nominal speed in $ [8m/s, 10m/s]$. For all the policies, we use a prediction horizon of $N=10$, a discretization time-step of $dt= 0.2 \text{ s}$, and a risk level of $\epsilon=0.02$ for the chance constraints in the SMPC.


The performance metrics, averaged across the initial conditions, are shown in Table~\ref{tab:comparison_scenario_all}.  In terms of mobility, \textbf{Proposed} is able to improve or maintain mobility compared to the ablations. There is a noticeable improvement in comfort and safety metrics, as the \textbf{Proposed} can stay close to the TV-free reference trajectory without incurring high acceleration/jerk and keeping a safe distance from the TV. \textbf{Proposed} was also able to find feasible solutions for the SMPC optimization problem more often our experiments because the formulation optimizes over policies and risk levels for the multi-modal constraints. Finally, we see that the \textbf{OL} is the fastest in solve time (because of the missing policy and risk variables). However, we see that introducing the additional risk level variables only marginally increases the solve time on comparing \textbf{Proposed} and \textbf{Fixed risk}. The higher solve times for the proposed approach in the lane change scenario are because of the additional TV and its associated multi-modal predictions. To remedy this issue for the hardware experiments in the next section, we use a multi-threaded implementation as described in Fig.~\ref{fig:planner_architecture} to solve the SMPC. 

The results highlight the benefits of optimizing over policies and incorporating the variable risk formulation in the SMPC formulation \eqref{opt:SMPC_skeleton} for the EV, towards collision avoidance with multi-modal predictions of the TV.

\section{Experimental Validation}
\label{sec:expt_design}
In this section, we validate our approach in hardware vehicle experiments to assess the benefits of our proposed stochastic MPC formulation. The experiment videos can be accessed at \url{https://shorturl.at/ctQ57}.

\subsection{Test scenario and key takeaways}
In the hardware experiment, we consider the same lane-change scenario introduced in Sec. \ref{ssec:sim_2}, wherein EV initiates a lane change maneuver with a leading TV ahead of the EV in the same lane and a trailing TV behind the EV in the adjacent lane as illustrated in Fig. \ref{fig:experiemnt_rfs_situation}.
In this scenario, EV predicts multi-modal behaviors of other TVs and tries to minimize the risk of collisions for every possible mode. As illustrated in Fig. \ref{fig:experiemnt_rfs_situation}, EV predicts two different modes of TVs: keeping their lanes or changing lanes. However, TVs will not change their lane until the end of this scenario.

We compare the proposed SMPC (\textbf{Proposed}) with the \textbf{OL}, introduced in Sec. \ref{ssec:sim_2}. 
The results in Fig. \ref{fig:x_cl_vs_ol} show that while the \textbf{OL} cannot find a feasible solution due to its conservativeness of the constraint tightening formulation, \textbf{Proposed} successfully accomplishes the given scenario without any collisions.
When \textbf{OL} problem becomes infeasible, we change the control policy \textbf{OL} to a backup control policy: keeping the current lane and decelerating mildly.
This abrupt policy change can deteriorate the comfort indices and it is more likely that tracking previous optimal trajectories from \textbf{OL} leads to smoother behaviors. However, due to safety concerns, we tried to stop the vehicle within the current lane. 

Furthermore, we study how the predicted mode probability of the leading TV affects the closed-loop behavior.
Compared to the case that the leading TV is likely to keep its lane ($p_\text{lk} = 0.9$), \textbf{Proposed} sets more margin in a lateral direction to avoid the collision in case the leading TV changes lanes ($p_\text{lk} = 0.1$) as illustrated in Fig. \ref{fig:x_prob_lklc}.
\begin{figure}[h]
    \centering    \includegraphics[width=0.70\columnwidth]{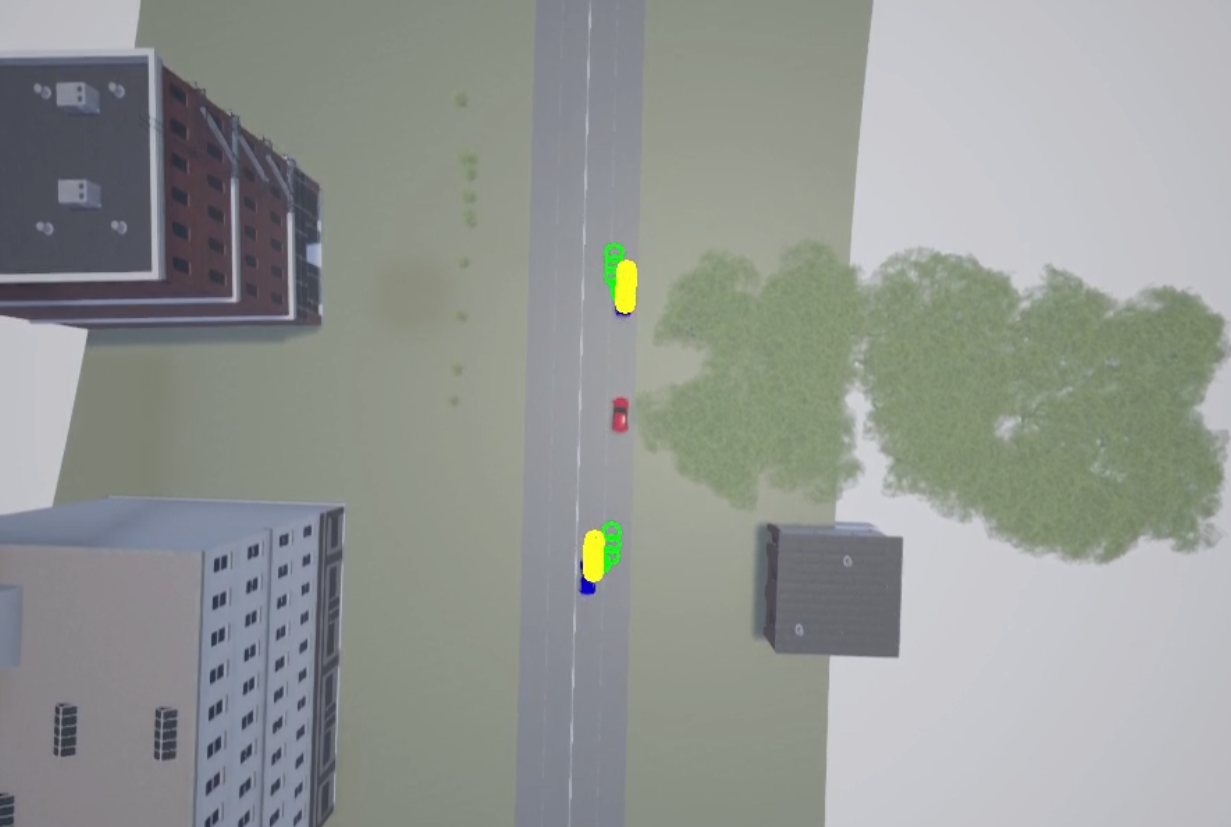}
    \caption{Drone view of the testing scenario including EV (red) and 2 TVs (blue) with predictions (yellow and green circles)}    \label{fig:experiemnt_rfs_situation}
\end{figure}

\subsection{Hardware Architecture for Experiments}


The test vehicle and the hardware setup are illustrated in Fig. \ref{fig:hardware_setup}.
The computing unit of the system consists of three computers: a Linux-based laptop, a Linux-based rugged computer, and the dSPACE MicroAutoBox II (MABXII).
The laptop is for simulating virtual environments and transmitting all information such as states of surrounding vehicles.
The rugged PC is for implementing a planning and control software stack that plans the ego vehicle's behavior, generates dynamically feasible, safe trajectories, and calculates acceleration and yaw rate to track the generated trajectories.
The MABXII is for implementing an actuator-level controller that calculates actuator control inputs and a fail-safe logic that provides safety features. 
\begin{figure}[h]
    \centering    \includegraphics[width=0.8\columnwidth]{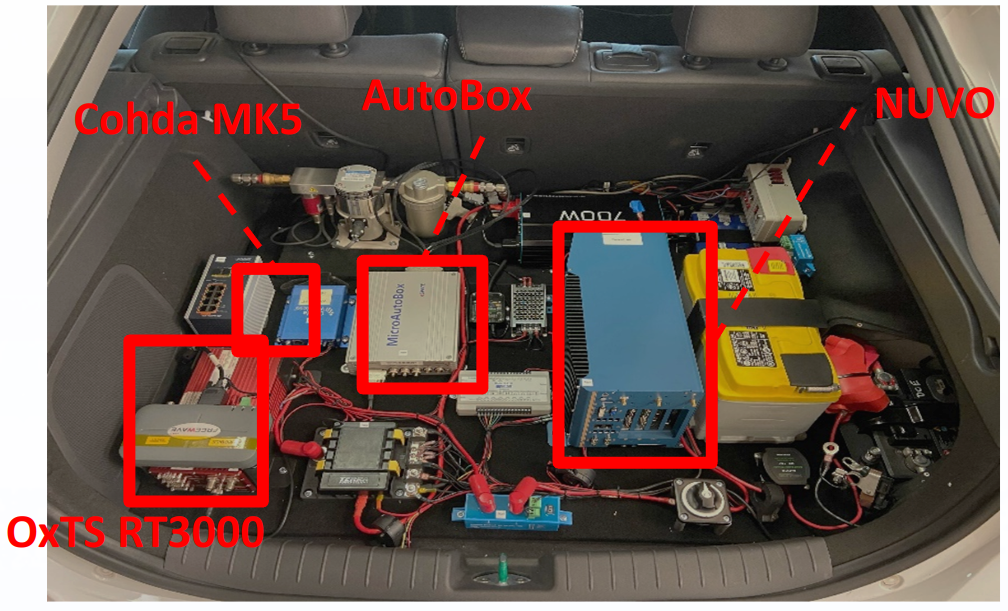}
    \caption{The hardware setup in the actual vehicle}
    \label{fig:hardware_setup}
\end{figure}

The sensors of the system are an OxTS RT3000: a differential GPS to localize the ego vehicle and production vehicle sensors to acquire vehicle state information. 

\subsection{Software Architecture for Experiments}
The overall block diagram of control architecture and the entire system is illustrated in Fig. \ref{fig:control_architecture}.
In the following subsections, we describe the comprehensive details of this experimental setup.
\begin{figure}[h!]
    \centering
    \vspace{-1.0em}
    \includegraphics[width=0.80\columnwidth]{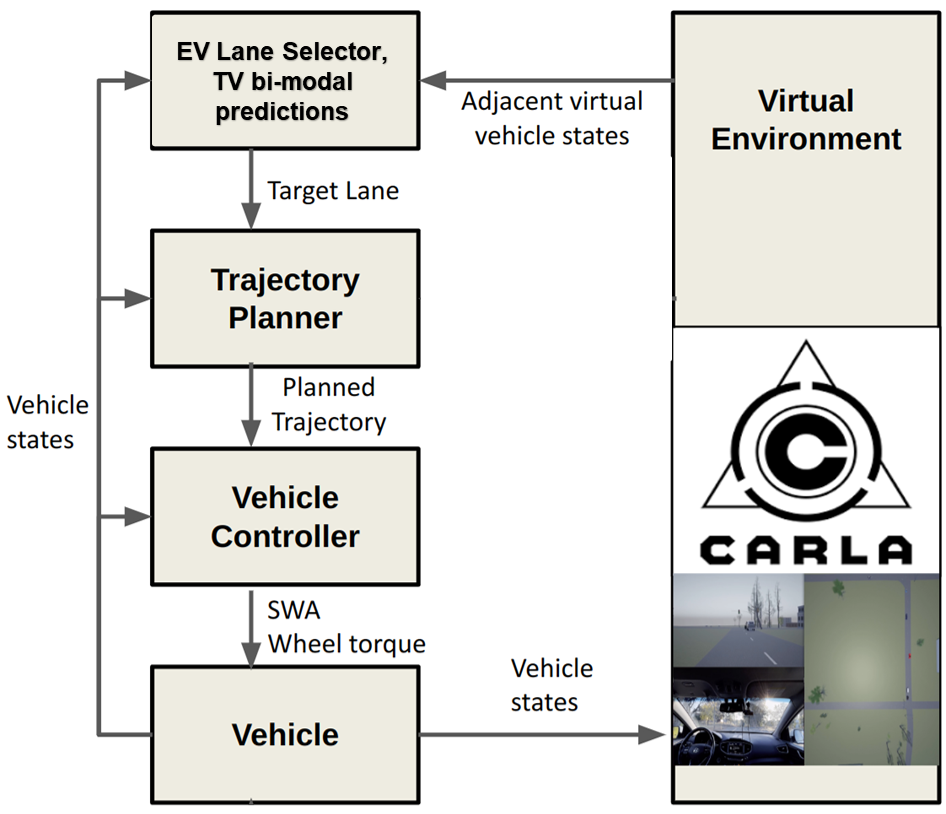}
    \caption{Diagram of Control Architecture}
    \vspace{-1.0em}
    \label{fig:control_architecture}
\end{figure}
\subsubsection{Planning and Control Software Stack}
The developed hierarchical control system consists of a lane selector, a trajectory planner, and a vehicle controller. 
First, the lane selector determines the target lane for the ego vehicle.
Second, given the target lane, the trajectory planner \cite{joa2023energy} generates a smooth, comfortable trajectory to the target lane by solving a nonlinear optimization problem. 
The calculated trajectory serves as the reference \eqref{eq:EV_ref_traj} for the proposed SMPC. Then the proposed SMPC \eqref{opt:SMPC}  is solved to determine acceleration and yaw rate commands that satisfy state/input constraints. A kinematic bicycle model \eqref{eq: bicycle model} is employed for the EV predictions in the SMPC optimization problem, which is modelled with CasADi and solved using Gurobi. For computing the SMPC commands at 10 Hz, we employ a multi-threaded architecture that computes the gains $K^{j}_{k|t}, M^j_{l,k|t}$ and feedforward terms $h^{j}_{k|t}$ at different frequencies as depicted in Fig.~\ref{fig:planner_architecture}.
\begin{figure}[h!]
    \centering
    \vspace{-1.0em}
\includegraphics[width=0.70\columnwidth]{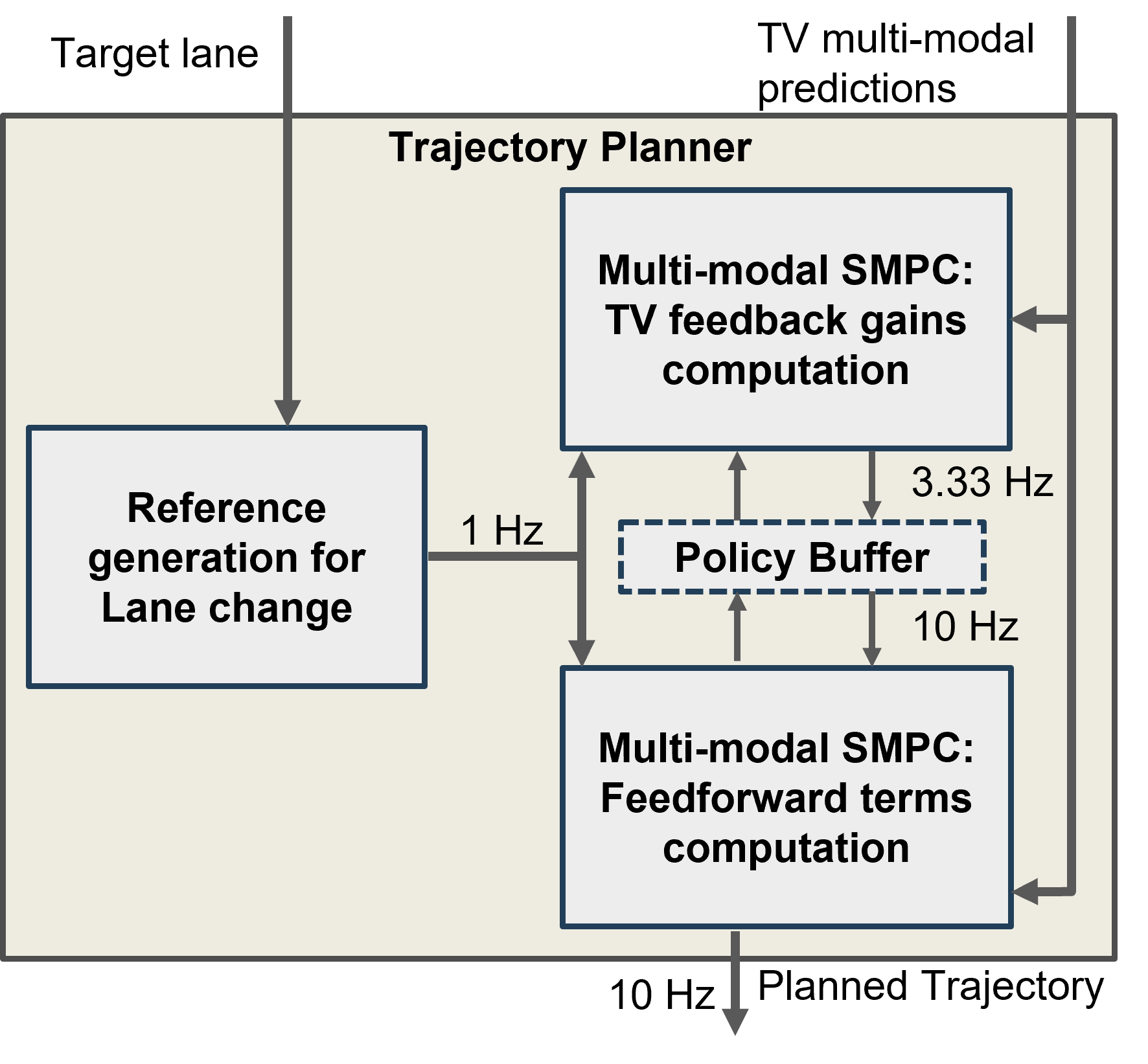}
\vspace{-1.0em}
    \caption{Trajectory planner architecture}
    \vspace{-1.0em}
    \label{fig:planner_architecture}
\end{figure}
Finally, the vehicle actuator controller calculates steering wheel angle and wheel torque commands from the optimal inputs of the proposed SMPC \eqref{opt:SMPC}.

\subsubsection{Virtual Environment Simulator (Digital Twin)}
To conduct real-world vehicle tests safely and efficiently, we employ a Vehicle-In-the-Loop (VIL) system as outlined in \cite{joa2023energy}, which integrates the operation of an actual vehicle with a virtual environment simulation.

The CARLA \cite{Dosovitskiy17} software is the primary simulator to build virtual environments and simulate a variety of scenarios with ease. 
The virtual environment simulator constructs all components such as road networks, other vehicles, traffic infrastructures, buildings, and so on to replicate the real-world map. Fig. \ref{fig:compare_map} shows the generated CARLA map, the satellite image of the testing site, and the actual test vehicle.

\begin{figure}[h]
    \centering    \includegraphics[width=0.85\columnwidth]{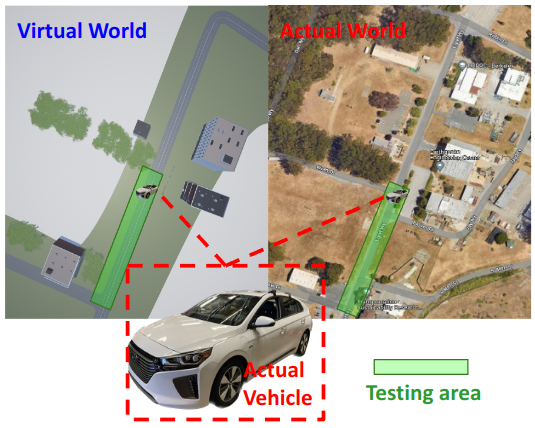}
    \caption{The CARLA image, the satellite image of the testing site (RFS) and the actual vehicle image}
    \label{fig:compare_map}
\end{figure}

On the customized map, the CARLA simulates a traffic scenario with the same initial conditions such as the number of spawned vehicles, the locations of the vehicles, etc. It is worth noting that the CARLA simulator exhibits inherent randomness in the motion of each virtual vehicle, resulting in variations in the resulting traffic scenario. 
We also synchronize the real world with the virtual world in terms of the physical EV. Based on the obtained coordinate data from dGPS/IMU sensors, the simulator generates an agent in the virtual world and teleports the vehicle by updating the position and orientation of the agent every time it receives data from the actual sensors.


\subsection{Experiment Results}
Our experiment setup parameters are described as follows. The EV initiates motion with an initial speed of zero. The leading TV begins its movement 10 meters ahead of the EV in the same lane, while the trailing TV starts from a position 25 meters behind the EV in the adjacent lane. Both Target Vehicles maintain a consistent average speed of $4$m/s. Each of the TVs operates in two distinct modes: Lane Keeping (LK) and Lane Change (LC). The trailing TV equally splits its mode probability, with a 50\% chance of LK and a 50\% chance of LC. To investigate the impact of varying mode probabilities, we change the probability of the mode chosen by the leading TV.

The conducted tests are executed at the Richmond Field Station (RFS), as illustrated in Fig. \ref{fig:compare_map}. The hardware experiments are primarily divided into two segments. The first segment involves a comparison between our approach (which optimizers over multi-modal policies and risk levels) against a baseline that only optimizes over multi-modal open-loop sequences with fixed risk levels for each mode (this is \textbf{OL} from Section \ref{ssec:sim_2}). The second segment focuses on assessing the behavioral outcomes resulting from alterations in the probability of the surrounding vehicle's lane change mode and lane keeping mode.

\subsubsection{\textbf{Proposed} vs \textbf{OL}}
Within the identical scenario, we conduct testing using two distinct control policies: \textbf{Proposed} in Sec. \ref{ssec:sim_2} and \textbf{OL} in Sec. \ref{ssec:sim_2}.
When the SMPC problem becomes infeasible, a lane-keeping controller with mild braking takes over the control. 
In an ideal practical application, a human driver should take over the control but due to our hardware limitations, we utilize the lane-keeping controller as a backup controller.

Fig. \ref{fig:x_cl_vs_ol} presents the closed-loop behaviors: a lateral error $e_y$ and a heading error $e_{\psi}$ with respect to a centerline, vehicle speed, steering wheel angle, and longitudinal acceleration. 
The graphs clearly illustrate that the \textbf{OL} yields infeasible solutions, leading to abrupt and undesirable motions with constraint violations. Conversely, the \textbf{Proposed} consistently generates feasible solutions, facilitating smooth motions in accordance with the predefined constraints.

\begin{figure}[h!]
    \centering
    \includegraphics[width=0.9\columnwidth]{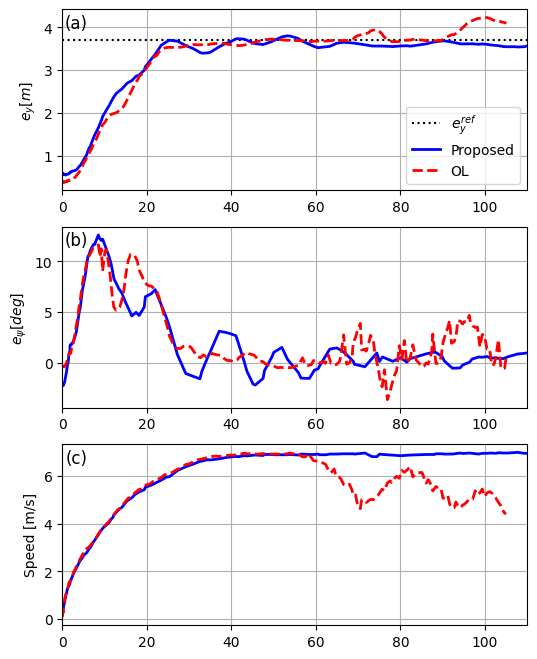}
    \\
\includegraphics[width=0.9\columnwidth]{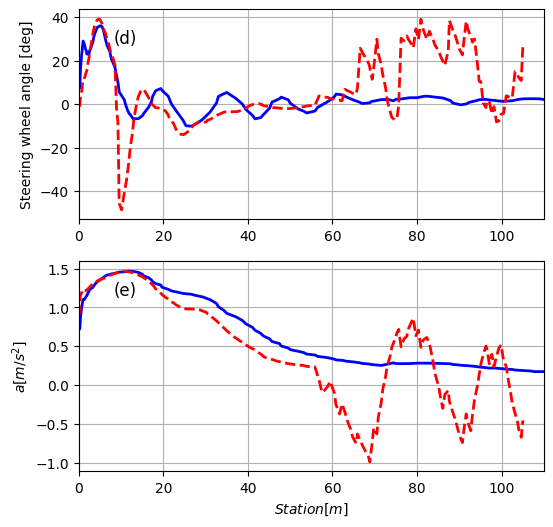}
    \caption{{\small{Comparison: \textbf{Proposed} vs \textbf{OL}. (a) Lateral error with respect to the centerline of the original lane. $e_y^\text{ref}$ denotes the reference, (b) Heading error with respect to the centerline, (c) Vehicle longitudinal speed, (d) Steering wheel angle, and (e) Longitudinal acceleration.
    \textbf{Proposed} makes the ego vehicle keep the lateral distance ($e_y$) close to the reference while satisfying the multi-modal collision avoidance constraints. In contrast, the \textbf{OL} becomes infeasible during the task.}}}
    \vspace{-2.0em}
    \label{fig:x_cl_vs_ol}
\end{figure}


\subsubsection{Change of mode probability}
In the second part, we proceed with testing under the \textbf{Proposed} while varying the lane-keeping probability of the leading TV. 
Specifically, we compare the case that the leading TV is likely to keep its lane ($p_\text{lk} = 0.9$) with the case that the leading TV changes lanes ($p_\text{lk} = 0.1$).
In Fig. \ref{fig:x_prob_lklc}'s lateral error ($e_y$) graph, it is evident that the lateral distance is greater in the scenario with a lower lane-keeping probability. This observation aligns with the intuitive analysis that the resulting control policy prioritizes the lane change maneuver of the leading vehicle to the lane currently occupied by the EV. Due to the presence of collision avoidance constraints, the EV endeavors to evade the anticipated trajectory of the leading TV by maintaining larger lateral safety margins.

\begin{figure}[h!]
    \centering
    \includegraphics[width=0.9\columnwidth]{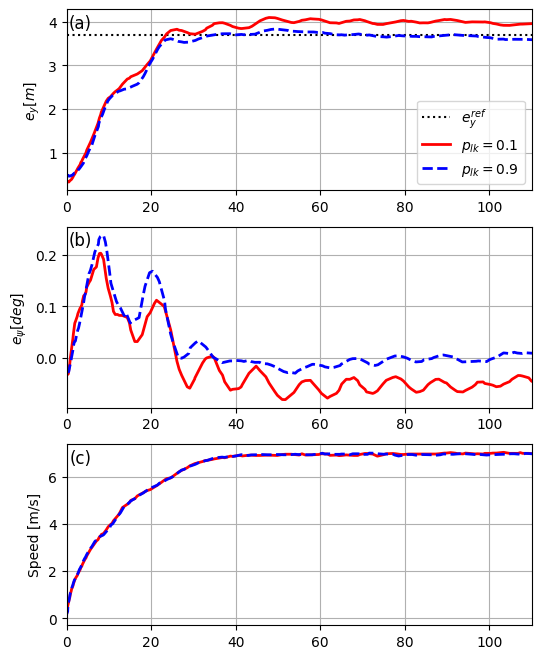}
    \\
    \includegraphics[width=0.9\columnwidth]{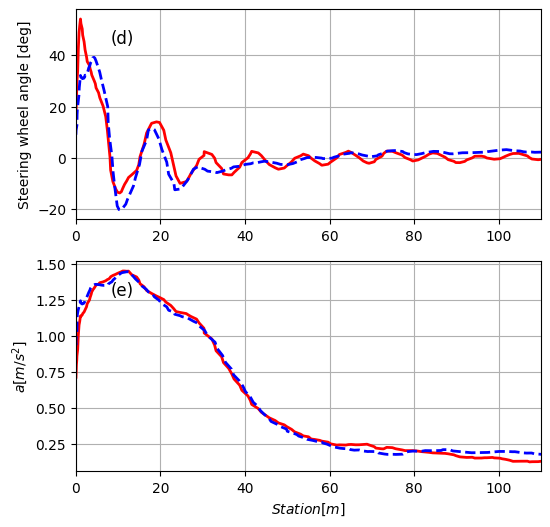}
    \caption{\small{Comparison: $p_\text{lk} = 0.1$ vs $p_\text{lk} = 0.9$}. (a) Lateral error with respect to the centerline of the original lane. $e_y^\text{ref}$ denotes the reference, (b) Heading error with respect to the centerline, (c) Vehicle longitudinal speed, (d) Steering wheel angle, and (e) Longitudinal acceleration. Compared to the case that the leading TV is likely to keep its lane ($p_\text{lk} = 0.9$), \textbf{Proposed} sets more margin in a lateral direction to avoid the collision in case the leading TV changes lanes ($p_\text{lk} = 0.1$)}
    \vspace{-2.0em}
    \label{fig:x_prob_lklc}
\end{figure}


\section{Conclusion}
\label{sec:conclusion}
We proposed a Stochastic MPC formulation for autonomous driving with multi-modal predictions of surrounding vehicles. We provide a convex formulation for simultaneously (1) optimizing over parameterized feedback policies and (2) allocating risk levels to each mode for multi-modal chance constraint satisfaction. This enhances the feasibility and closed-loop performance of the SMPC algorithm, as demonstrated by our simulations and hardware experiments. 

\bibliographystyle{IEEEtran}
\bibliography{references.bib}

\appendix
\subsection{Matrix definitions}\label{app:matrices}
  \begingroup
  \allowdisplaybreaks
  \footnotesize 
\begin{align}
    &\mathbf{h}_t^j=[h^{j\top}_{t|t}\dots  h^{j\top}_{t+N-1|t} ]^\top\label{mat:h}\\
&\mathbf{K}_t^j=\text{blkdiag}\left(K^j_{t|t},\dots, K^j_{t+N-1|t}\right)\label{mat:K}\\
    &\mathbf{M}_t^{j}=\begin{bmatrix}
    O&\hdots&\hdots&\hdots& O\\M^j_{t,t+1|t}&O&\hdots&\hdots& O\\
    \vdots&\vdots&\vdots&\vdots\\
    M^j_{t, k|t}&\hdots M^j_{k-1, k|t}&O&\hdots& O\\
    \vdots&\vdots&\vdots&\vdots\\
    M^{j}_{t,t+N-1|t}&\hdots &\hdots& M^{j}_{t+N-2,t+N-1|t}&O
    \end{bmatrix}\label{mat:M}\\
    &\mathbf{A}_t=\begin{bmatrix}I_4\\ A_{t|t}\\\vdots\\\prod\limits_{k=t}^{t+N-1}A_{k|t}\end{bmatrix},  \mathbf{B}_t=\begin{bmatrix}O&\hdots&\hdots& O\\B_{t|t}&O&\hdots&O\\\vdots&\ddots&\ddots&\vdots\\\prod\limits_{k=t+1}^{t+N-1}A_{k|t}B_{t|t}&\hdots&\dots&B_{t+N-1|t}\end{bmatrix},\label{mat:AB}\\
    &\mathbf{T}^j_t=\begin{bmatrix}I_2\\ T_{t|t,j}\\T_{t+1|t,j}T_{t|t,j}\\\vdots\\\prod\limits_{k=t}^{t+N-1}T_{k|t,j}\end{bmatrix}, \mathbf{C}^j_t=\begin{bmatrix}O\\c_{t|t,j}\\c_{t+1|t,j}+T_{t+1|t,j}c_{t|t,j}\\\vdots\\c_{t+N-1|t,j}+\sum\limits_{k=t}^{t+N-1}\prod\limits_{l=k+1}^{t+N-1}T_{l|t,j} c_{k|t,j}\end{bmatrix}\label{mat:TC}\\
    &\begin{aligned}\mathbf{E}_{t}=\begin{bmatrix}O&\hdots&\hdots& O\\I_4&O&\hdots&O\\A_{t+1|t}&I_4&\hdots&O\\\vdots&\ddots&\ddots&\vdots\\\prod\limits_{k=t+1}^{t+N-1}A_{k|t}&\hdots&\dots&I_4\end{bmatrix},\\  \mathbf{L}^j_{t}=\begin{bmatrix}O&\hdots&\hdots& O\\I_2&O&\hdots&O\\T_{t+1|t,j}&I_2&\hdots&O\\\vdots&\ddots&\ddots&\vdots\\\prod\limits_{k=t+1}^{t+N-1}T_{k|t,j}&\hdots&\dots&I_2\end{bmatrix}\end{aligned} \label{mat:EL}\\
    &\boldsymbol{\Sigma}_w=I_N\otimes \Sigma_w,\  \boldsymbol{\Sigma}^j_n=\text{blkdiag}(\Sigma_{t|t,j},\dots,\Sigma_{t+N-1|t,j}) \label{mat:Sigma}\\
    &\mathbf{Q}=I_{N+1}\otimes Q,\ \mathbf{R}=I_N\otimes R,\label{mat:QR}
    \end{align}
\normalsize 
\endgroup

\end{document}